%% file: note-ef.tex
\documentclass[10.5pt,letterpaper]{article}

\usepackage{graphicx} %
\input{epfml_macros.tex}
\usepackage{fullpage}
\usepackage{booktabs}

\let\cite\citep

\usepackage{threeparttable}
\renewcommand{\tnote}[1]{\textsuperscript{\textbf{#1}}}

\usepackage{algorithm}
\usepackage[noend]{algpseudocode}

\newcommand{\update}{}

\errorcontextlines\maxdimen

\makeatletter
    \newcommand*{\algrule}[1][\algorithmicindent]{\makebox[#1][l]{\hspace*{.5em}\thealgruleextra\vrule height \thealgruleheight depth \thealgruledepth}}%
\newcommand*{\thealgruleextra}{}
\newcommand*{\thealgruleheight}{.75\baselineskip}
\newcommand*{\thealgruledepth}{.25\baselineskip}

\newcount\ALG@printindent@tempcnta
\def\ALG@printindent{%
    \ifnum \theALG@nested>0%
        \ifx\ALG@text\ALG@x@notext%
        \else
            \unskip
            \addvspace{-1pt}%
            \ALG@printindent@tempcnta=1
            \loop
                \algrule[\csname ALG@ind@\the\ALG@printindent@tempcnta\endcsname]%
                \advance \ALG@printindent@tempcnta 1
            \ifnum \ALG@printindent@tempcnta<\numexpr\theALG@nested+1\relax%
            \repeat
        \fi
    \fi
    }%
\usepackage{etoolbox}
\patchcmd{\ALG@doentity}{\noindent\hskip\ALG@tlm}{\ALG@printindent}{}{\errmessage{failed to patch}}
\makeatother

\newbox\statebox
\newcommand{\myState}[1]{%
    \setbox\statebox=\vbox{#1}%
    \edef\thealgruleheight{\dimexpr \the\ht\statebox+1pt\relax}%
    \edef\thealgruledepth{\dimexpr \the\dp\statebox+1pt\relax}%
    \ifdim\thealgruleheight<.75\baselineskip
        \def\thealgruleheight{\dimexpr .75\baselineskip+1pt\relax}%
    \fi
    \ifdim\thealgruledepth<.25\baselineskip
        \def\thealgruledepth{\dimexpr .25\baselineskip+1pt\relax}%
    \fi
    \State #1%
    \def\thealgruleheight{\dimexpr .75\baselineskip+1pt\relax}%
    \def\thealgruledepth{\dimexpr .25\baselineskip+1pt\relax}%
}

\commenter{seb}{red}

\title{On Communication Compression for Distributed Optimization on Heterogeneous Data}
\author{Sebastian U. Stich\thanks{\texttt{sebastian.stich@epfl.ch}, Machine Learning and Optimization Lab (MLO), EPFL, Switzerland.}\\ EPFL 
}
\date{}

\begin{document}
\maketitle

\begin{abstract} 
Lossy gradient compression, with either unbiased or biased compressors, has become a key tool to avoid the communication bottleneck in centrally coordinated distributed training of machine learning models. 
We analyze the performance of two standard and general types of methods:\ (i) distributed quantized SGD (D-QSGD) with arbitrary unbiased quantizers and 
(ii) distributed SGD with error-feedback and biased compressors (D-EF-SGD) in the heterogeneous (non-iid) data setting.

Our results indicate that D-EF-SGD is much less affected than D-QSGD by non-iid data, but both methods can suffer a slowdown if data-skewness is high. We {\update further study} two alternatives that are not (or much less) affected by heterogenous data distributions: first, {\update a recently proposed} method that is effective  on strongly convex problems, and secondly, we point out a more general approach that is applicable to linear compressors only but effective in all considered scenarios.
\end{abstract}

\section{Introduction}
We consider the distributed optimization problem
\begin{align}
f^\star := \min_{\xx \in \R^d} \sbr*{f(\xx):= \frac{1}{n}\sum_{i \in [n]} f_i(\xx)}\,, \label{eq:main}
\end{align}
where the objective function $f \colon \R^d \to \R$ is split among $n$ terms $f_i \colon \R^d \to \R, i \in [n]$, that are distributed among $n$ nodes. We assume that $f$ is $L$-smooth and that we have access to unbiased gradient oracles with $\sigma^2$-bounded variance for each $f_i, i \in [n]$. We study the heterogeneous setting and allow skewed data distributions on the nodes. We quantify the data-dissimilarity by a parameter $\zeta^2 \geq 0$.

Synchronous parallel SGD and variants thereof~\cite[e.g.][]{duchi2011adagrad,Kingma2014:adam} are among the most popular optimization algorithms in machine- and deep-learning~\cite{Bottou2010:sgd}. Because the number of parameters in neural networks can we very large, the time required to share the gradients across workers limits the
scalability of deep learning training 
\cite{Seide2015:1bit,Strom:2015wc}.
To address this bottleneck, lossy gradient compression techniques have been proposed as a solution, for instance \cite{Seide2015:1bit,Alistarh2017:qsgd,%
Wen2017:terngrad,Bernstein2018:sign}.

Whilst many empirical works highlighted the importance of data-adaptive compressors \cite{Lin2018:deep,Alistarh2018:topk,Wangni2019:tng,%
vogels2019}
that adapt to the local data distribution on the nodes, 
many theoretical analyses did often not consider the heterogeneous setting so far.
In this note, we refine the analyses in \cite{Alistarh2017:qsgd,Cordonnier2018thesis} and show how two commonly used training schemes are impacted by (potentially) non-iid data distributions on the nodes.

We consider two classes of methods:\ distributed methods with 
(i) unbiased gradient compressors (denoted as D\nobreakdash-QSGD in the following), with QSGD~\cite{Alistarh2017:qsgd}, Terngrad~\cite{Wen2017:terngrad} and signSGD~\cite{Bernstein2018:sign} as a few representative members, and
distributed methods with (ii) biased compressors and error-feedback (denoted by D-EF-SGD), such as proposed in \cite{Seide2015:1bit,Alistarh2018:topk,StichCJ18sparseSGD}.

As our first contribution, we tighten the existing analyses of these two types of methods and provide analyses for general non-convex, convex and strongly-convex problems. Exemplary, for instance for the case of $\mu$-strongly convex functions, we show that these methods converge as
\begin{align*}
&\text{D-QSGD type}: & &\tilde \cO \rbr*{\frac{\sigma^2 + \zeta^2}{\mu n \delta \epsilon } + \rbr*{\frac{L}{\mu} + \frac{L}{\delta n}}\log \frac{1}{\epsilon}}\\
&\text{D-EF-SGD type}: & & \tilde \cO \rbr*{\frac{\sigma^2}{\mu n  \epsilon } +\frac{(\sigma + \zeta/\sqrt{\delta})\sqrt{L}}{\mu \sqrt{\delta \epsilon}} + \frac{L}{\mu \delta} \log\frac{1}{\epsilon}} & & \qquad \qquad
\end{align*}
where here $0 < \delta \leq 1$ is a parameter measuring the compression quality ($\delta=1$ meaning no compression and recovering the standard SGD convergence rates).
In the presence of high stochastic noise, $\sigma^2 \ggg1$, D-QSGD methods suffer from a linear slow-down with respect to the the compression quality $\delta$, whereas for D-EF-SGD methods the first term is not affected by $\delta$. This characteristic performance difference has been discussed in prior work~\cite[e.g.][]{StichCJ18sparseSGD,StichK19delays} and we here show in addition that D-EF-SGD methods are also less sensitive to data-skewness.

In a slightly stronger setting, under the additional assumption that the local functions $f_i$ are smooth and convex, \citet{Mishchenko2019:diana} proposed the DIANA framework that is even less sensitive to data-skewness. In particular, DIANA converges linearly in the special case when $\sigma^2 = 0$, in contrast to the two methods introduced above. However, the dependency on $\sigma^2$ is not optimal in this scheme.
{\update We study the convergence of D-EF-SGD with bias correction} that converges as
\begin{align*}
\text{D-EF-SGD with bias correction}: & \qquad\qquad \tilde \cO \rbr*{\frac{\sigma^2}{\mu n  \epsilon } + \frac{\sigma \sqrt{L}}{\mu \sqrt{\delta \epsilon} }  + \rbr*{\frac{1}{\delta^2} + \frac{L}{\mu \delta} }\log\frac{1}{\epsilon}}
\end{align*}
This rate depends only poly-logarithmically on the data-dissimilarity parameter $\zeta^2$ (hidden in the $\tilde\cO(\cdot)$ notation) and converges linearly in the special case when $\sigma^2 = 0$.
{\update The algorithm that we study here is a minor variation of a scheme proposed earlier by~\cite{Kovalev2019:personal}.\footnote{\update 
Their scheme, EC-SGD-DIANA (see Algorithm~\ref{alg:ec-diana}), was---to the best of our knowledge at the time of writing this manuscript---known to converge at rate $\tilde\cO\big(\frac{\sigma^2}{\mu^2 n \epsilon} +  \frac{L n }{\mu \delta^2} \big)$ \cite{Kovalev2019:personal}. In parallel work, \citet{Gorbunov2020:ef} recently provided a tighter analysis of EC-SGD-DIANA and recover and extend our results presented here.
}}

We further point out an important observation, that when using linear compressors the convergence rate
\begin{align*}
\text{D-EF-SGD with linear compressors}: & \qquad\qquad \tilde \cO \rbr*{\frac{\sigma^2}{\mu n  \epsilon }  + \frac{L}{\mu \delta} \log\frac{1}{\epsilon}}
\end{align*}
can be obtained, which does not depend on the data-skewness. Whilst this approach requires additional restriction on the amenable compressors, it does not require additional assumptions on the regularity of the objective function and works also for the convex and non-convex case.

\section{Related Work}
Communication compression
is an established approach to alleviate the communication bottleneck in parallel optimization for deep-learning and a variety of different compressors have been proposed and studied~\cite{Seide2015:1bit,
Alistarh2017:qsgd,%
Aji2017:topk,Wen2017:terngrad,Zhang2017:zipml,%
Bernstein2018:sign,Wangni2018:sparsification}.
It has been demonstrated that the application of these methods is not limited to parallel SGD implementations alone, but can be combined e.g.\ with variance reduction~\cite{kuenstner2017:svrg} or with communication over arbitrary network topologies~\cite{Tang2018:decentralized}.
The analyses of D-QSGD in \cite{Alistarh2017:qsgd} for the stochastic case ($\sigma^2 > 0$), and in~\cite{Khirirat2018:dqsgd} for the deterministic ($\sigma^2=0$) case the most closely related works, which both did not consider the data-dissimilarity parameter in their analysis.

The observed practical successes of error-feedback mechanisms (that compensate compression errors), such as in~\cite{Seide2015:1bit}, could be theoretically explained in~\cite{StichCJ18sparseSGD,Alistarh2018:topk,%
KarimireddyRSJ2019feedback,StichK19delays}. Error-feedback mechanism have been successfully applied for different compressors~\cite{Ivkin2019:sketch,vogels2019} or different settings, such as decentralized~\cite{KoloskovaSJ19gossip,Tang2019:squeeze,KoloskovaLSJ19decentralized} or federated learning~\cite{Rothchild2020:fed}.
The first analyses for the multiple worker case were given in~\cite{Alistarh2018:topk,Cordonnier2018thesis} and refined in~\cite{Beznosikov2020} by considering the heterogeneous setting. Our results improve over these prior works as we explain in more detail below.

The DIANA method was proposed by~\citet{Mishchenko2019:diana} to address distributed training with communcation compression for problems with non-smooth regularizers. In DIANA, gradient differences rather than iterates or gradients are quantized, similar as in~\cite{kuenstner2017:svrg}.
 In this work we follow closely the analysis presented in~\cite{Horvath2019:vr}.

Whilst for centralized parallel SGD the data-dissimilarity between the local objective functions does not affect the performance of SGD~\cite{Bottou2018:book}, it has been observed for instance in federated learning (where methods are allowed to perform several local gradient steps before synchronization) or in decentralized optimization (where methods typically only use imperfect synchronization in each round), data-skewness heavily impacts the performance of most standard training schemes~\cite[cf.][]{%
hsieh2019quagmire,
Kgoogle:cofefe,%
Koloskova20decentralized,Li2020:fedavg,%
woodworth2020:heterogenous}.

\section{Assumptions}
We now list the main assumptions on the optimization problem~\eqref{eq:main}. For simplicity and the ease of presentation, we focus here on the most common standard assumptions, but the analyses could be tightened for many special cases, following techniques developed in other works.

\subsection{Regularity assumptions}
For all our results, we assume $L$-smoothness of $f$:
\begin{align}
 \norm{\nabla f(\yy)- \nabla f(\xx)} &\leq L \norm{\yy - \xx}\,, &  &\forall \xx,\yy \in \R^d\,. \label{def:lsmooth}
\end{align}
For some results we further require that each $f_i$, $i \in [n]$ is $L$-smooth. This assumption could for instance be relaxed by considering different smoothness constants $L_i$, $i \in [n]$.

Sometimes we require $\mu$-strong convexity of $f$ (or just convexity for $\mu = 0$):
\begin{align}
 f(\xx) &\geq f(\yy) + \lin{\nabla f(\yy),\xx-\yy} + \frac{\mu}{2}\norm{\xx-\yy}^2\,,  & &\forall \xx,\yy \in \R^d\,. \label{def:convex} 
\end{align}
For some results we require in addition each $f_i$, $i \in [n]$ to be convex. The convexity assumption can for most of the results be relaxed to star-convexity~\cite[cf.][]{StichK19delays} instead, or by assuming the Polyak-\L{}ojasiewicz condition~\cite[cf.][]{Karimi2016:pl}.

\subsection{Assumption on noise}
We assume that we have access to stochastic gradient oracles $\gg^i(\xx) \colon \R^d \to \R^d$ for each component $f_i$, $i \in [n]$. For simplicity we only consider the instructive case of uniformly bounded noise:
\begin{align}
\gg^i(\xx)& = \nabla f_i(\xx) + \bxi^i\,, &  \E_{\bxi^i} { \bxi^i } &= \0_d\,, &  \E_{\bxi^i} \norm{\bxi^i}^2 \leq \sigma^2\,, & &\forall \xx \in \R^d, i \in [n]\,. \label{def:noise}
\end{align}
With techniques introduced in other works, one can for instance extend the analysis to variations when the noise is assumed to scale with the squared norm of the gradient~\cite[cf.][]{Bottou2018:book,Stich19sgd}, or function suboptimality gap~\cite[cf.][]{Khaled2020:sgd}. Under additional structural assumptions, for instance assuming that each $f_i$ is $L$-smooth, or that each stochastic gradient is the gradient of a smooth function $\gg^i(\xx)=\nabla F(\xx,\bxi^i)$, additional tightening of the results can be obtained (such as for instance replacing $\sigma^2$ in the rates by a bound on the noise $\sigma_\star^2$ at the optimum $\xx_\star$ only).

\subsection{Gradient Dissimilarity}
In this work we consider the heterogeneous data setting and allow the functions $f_i$, $i \in [n]$ to be different on each node. We measure dissimilarity by two constants $\zeta^2 \geq 0$, $Z \geq 1$ that bound the variance across the $n$ nodes:
\begin{align}
 \frac{1}{n} \sum_{i=1}^n \norm{\nabla f_i(\xx)}^2 \leq \zeta^2 + Z^2 \norm{\nabla f(\xx)}^2\,. \label{def:zeta}
\end{align}
This is similar to the assumption in \cite{Koloskova20decentralized}. 
For the special case of $Z=1$ this matches the notions in related works~\cite[such as][]{Mishchenko2019:diana,Vogels2020:power}, but allowing $Z \geq 1$ is slightly more general. 
Whilst in principle we could also allow $Z \in [0,1]$ (at the expense of a larger $\zeta^2$),
we note that $\norm{\nabla f(\xx)}^2 \leq \frac{1}{n} \sum_{i=1}^n \norm{\nabla f_i(\xx)}^2$ and hence only $Z\geq 1$ allows for scale-free bounds in general (for instance, imposing $Z=0$ would imply an uniform bound on the gradient norms; an assumption which we do want to avoid here).
When assuming smoothness and convexity, it is often natural to measure dissimilarity only at the optimum $\xx_*$, denoted by a constant $\zeta_\star^2$.

\subsection{Compressors}
We introduce two notions of compressors that have become popular in the literature. To better distinguish them in this manuscript, we will use slightly different terms and parameters to denote them.

A $\delta$-compressor\cite[cf.][]{StichCJ18sparseSGD} is a mapping $\cC_\delta \colon \R^d \to \R^d$, with the property
\begin{align}
 \E_{\cC_\delta} \norm{\cC_\delta(\xx) -\xx}^2 &\leq (1-\delta)\norm{\xx}^2\,, & &\forall \xx \in \R^d\,. \label{def:compressor}
\end{align}

A $\omega$-quantizer~\cite[cf.][]{Alistarh2017:qsgd,kuenstner2017:svrg}, is a mapping $\cQ_\omega \colon \R^d \to \R^d$, with the property
\begin{align}
\E_{\cQ_\omega} \cQ_\omega(\xx) &= \xx\,, & \E_{\cQ_\omega} \norm{\cQ_\omega(\xx)}^2 &\leq (1+\omega)\norm{\xx}^2\,, & &\forall \xx \in \R^d\,. \label{def:quantizer}
\end{align}
Any $\omega$-quantizer can be rescaled to satisfy~\eqref{def:compressor}: $\frac{1}{1+\omega}\cQ_\omega$ is a $\delta=\frac{1}{1+\omega}$ compressor.

These notions do not guarantee a `compression' in the classical sense\footnote{For instance, $\xx \mapsto (1-\delta)\xx$ is a $\delta$-compressor, and $\xx \mapsto \{(1-\sqrt{\omega})\xx,(1+\sqrt{\omega})\xx\}$ with equal probability is a $\omega$-quantizer.}, but have been proven to be useful abstractions for the theoretical analysis of communication efficient SGD algorithms. Intuitively, we can assume that many compressors used in practice require approximately a $\delta$-fraction (or $\frac{1}{1+\omega}$ fraction, respectively) less bits compared to sending the full vector $\xx \in \R^d$, but this is not a rigorous statement.

An important (and illustrative) class of quantizers (or compressors) are sketching operators. As a guiding example, consider a linear sketch $\cS_\mV \colon \R^d \to \R^d $ of the form:
\begin{align}
 \cS_\mV(\xx) :=  \mV (\mV^\top \mV)^{-1} \mV^\top \xx \label{def:S}
\end{align}
for a matrix $\mV \in \R^{d \times p}$, $p \geq 1$. For instance, for $\mV=\ee_i$, a standard unit vector, this recovers random sparsification (when $\ee_i$ is chosen uniformly at random) or top-1 sparsification, when the index $i$ is chosen to match with the element of $\xx$ with largest magnitude. Both these operators are $\delta=\frac{1}{n}$ compressors, and the rescaled $n\cdot \cS_{\ee_i}(\xx)$ operator is a $\omega=n-1$ quantizer for a random choice of $\ee_i$, but not for the (biased) top-1 selection. These sketches can be (approximately) encoded only using $\cO(p\log(d) + B)$ bits at most, where $B$ denotes the bit length of a floating point number. These statements can be made more rigorous, but are not in the central focus here.

Popular sketching operators are for instance top-$k$ compressors~\cite{Aji2017:topk,Alistarh2018:topk,StichK19delays},
linear sketches~\cite{Konecny2016:federated},
count-sketches~\cite{Ivkin2019:sketch,Rothchild2020:fed} and
low-rank projections~\cite{vogels2019}.

\section{Distributed QSGD and Distributed EF-SGD}
In this section we derive new and improved convergence rates for the baseline algorithms D-QSGD and D-EF-SGD, tightening prior results in the literature. 
For instance, the analysis of D-QSGD in~\cite{Alistarh2017:qsgd} assumed a uniform bound on the gradient norms, $\E \norm{\gg_t^i}^2 \leq G^2$, $\forall i \in [n]$. This assumption can hide effects of non-iid data distributions across the nodes (as $\norm{\nabla f_i(\xx)}^2 \leq G^2$ is bounded). With our more general assumptions we are able to disentangle the two effects of the stochastic noise and the data-dissimilarity. 

All results, also for the following sections, are listed in Table~\ref{tab:results} for reference. In the main body of the text we only list the results for strongly convex functions for conciseness (and we do neither optimize nor compare constants in all the rates). All proofs can be found in the appendix.

\subsection{D-QSGD, Algorithm~\ref{alg:qsgd}}
Whilst variations of quantized SGD have been discussed in many early works or for special cases, a thorough theoretical discussion was provided in \cite{Alistarh2017:qsgd}, which popularized quantized SGD methods for efficient optimization in machine learning. Whilst their analysis required a uniform bound on the gradients, $\E \norm{\gg_t^i}^2 \leq G^2$, we do no require this assumption here. 
\citet{Khirirat2018:dqsgd} only study the case when $\sigma^2 = 0$ for a subset of loss functions that we consider here.
\begin{theorem}[D-QSGD]
\label{thm:qsgd}
Let $f\colon \R^d \to \R$ be $\mu$-strongly convex and $L$-smooth.
Then there exists a stepsize $\gamma \leq \frac{1}{2L(1+Z^2 \omega /n)}$ such that after at most
\begin{align}
T = \tilde \cO \rbr*{\frac{\sigma^2(1+\omega) + \zeta^2 \omega}{\mu n \epsilon} + \frac{L(1+Z^2\omega/n)}{ \mu}} \label{eq:bound_dqsgd}
\end{align}
iterations of Algorithm~\ref{alg:qsgd} it holds $\Ef f(\xx_{\rm out}) - f^\star \leq \epsilon$, where $\xx_{\rm out} = \xx_t$ denotes an iterate $\xx_t \in \{\xx_0,\dots,\xx_{T-1}\}$, chosen at random with probability proportional to $(1-\mu \gamma)^{-t}$.
\end{theorem}
\begin{remark}
We here state all convergence results for $\xx_{\rm out}$ chosen to be a random iterate. For convex functions this also implies convergence in function value of a weighted average of the iterates.
\end{remark}
\begin{remark}
Assuming $\sigma^2 \leq G^2, \zeta^2 \leq G^2$ for a constant $G^2$, we recover the $\cO \rbr[\big]{\frac{G^2 (1+\omega)}{\mu n \epsilon}}$ leading term derived in~\cite{Alistarh2017:qsgd}.
\end{remark}

\subsection{D-EF-SGD, Algorithm~\ref{alg:ef}}
Next, we consider distributed SGD with error-feedback.
Whilst the first analysis was presented in \cite{StichCJ18sparseSGD} only for the case $n=1$ and extended to $n>1$ in \cite{Cordonnier2018thesis}, both works assumed a uniform bound on the gradient norms. This assumption was revoked later in \cite{StichK19delays} for $n=1$ and in \cite{Beznosikov2020} for $n>1$. Our analysis improves over \cite[Theorem 15]{Beznosikov2020} in various aspects, for instance their result shows a dependence on $\smash{\cO\rbr[\big]{\frac{\sigma^2 + \zeta_\star^2/\delta}{\mu \epsilon}}}$ under the additional assumption that each $f_i$ is smooth and strongly convex, whilst we improve the respective terms to $\smash{\cO\rbr[\big]{\frac{\sigma^2}{\mu n \epsilon} + \frac{\zeta}{\mu \delta \sqrt{ \epsilon}}}}$ here under weaker assumptions, i.e.\ showing a linear speedup in $n$ for the leading term a and a weaker dependency on $\zeta^2$ (though $\zeta_\star^2 \leq \zeta^2$).
\begin{theorem}[D-EF-SGD]
\label{thm:ef}
Let $f\colon \R^d \to \R$ be $\mu$-strongly convex and $L$-smooth.
Then there exists a stepsize $\gamma \leq \frac{1}{14L(1+Z/\delta)}$ such that after at most
\begin{align}
 \tilde \cO \rbr*{\frac{\sigma^2}{\mu n \epsilon} + \rbr*{ \frac{L(\sigma^2 + \zeta^2/\delta)}{\mu^2 \delta \epsilon}}^{1/2} + \frac{L(1+Z/\delta)}{\mu}} \label{eq:bound_ef}
\end{align}
iterations of Algorithm~\ref{alg:ef} it holds $\Ef f(\xx_{\rm out}) - f^\star \leq \epsilon$, where $\xx_{\rm out} = \xx_t$ denotes an iterate $\xx_t \in \{\xx_0,\dots,\xx_{T-1}\}$, chosen at random with probability proportional to $\rbr*{1-\min\cbr*{\frac{\mu \gamma}{2},\frac{\delta}{4}}}^{-t}$.
\end{theorem}

\subsection{Discussion}

\paragraph{With stochastic noise.}
In the presence of stochastic noise $\sigma^2 > 0$, the first term is dominating in the rates when $\epsilon \to 0$. Due to mini-batching, this term decreases linearly in $n$ for both methods.
 We see that D-QSGD without error-feedback suffers from  a linear slow-down in $(1+\omega)$, $\smash{\cO \rbr[\big]{\frac{\sigma^2(1+\omega)}{\mu n\epsilon}}}$, whereas in D-EF-SGD the term $\cO \rbr[\big]{\frac{\sigma^2}{\mu n\epsilon}}$ is not affected by $\delta$. These characteristic effects and benefits of error compensation have been discussed in many prior works~\cite[cf.][]{StichCJ18sparseSGD,KarimireddyRSJ2019feedback}.

\paragraph{Without stochastic noise.}
For the special case when $\sigma^2=0$, we observe  that both D-QSGD and D-EF-SGD only converge sublinearly, at rates $\smash{\cO \rbr[\big]{\frac{\zeta^2 \omega}{\mu n \epsilon}}}$ and $\smash{\cO \rbr[\big]{\frac{\zeta}{\mu \delta \sqrt{\epsilon}}}}$, respectively.
 Despite that the parameter $\zeta$ can be zero for many applications (for instance for overparametrized optimization problems), these results show that 
data-dissimilarity impose additional challenges to optimization schemes with communication compression.

Qualitatively, the effects of the data-dissimilarity parameter on the convergence rate is similar as for local update methods that perform several local steps between communication rounds~\cite{Koloskova20decentralized}. This might just be a consequence of the (similar) proof techniques but might hint to an intrinsic limitation of the two approaches discussed in this section.

\begin{table}
\caption{Summary of the convergence results for $L$-smooth functions (with additional assumptions per column). $R_0^2 \geq \norm{\xx_0 -\xx_\star}^2$, $F_0 \geq f(\xx_0) - f^\star$.}
\label{tab:results}
\begin{threeparttable}
\resizebox{\linewidth}{!}{
\begin{tabular}{llrrr}
\toprule
Algorithm & compressor & $\mu$-strongly convex\tnote{a} & convex\tnote{b} & -\tnote{c} \\ \midrule
D-QSGD & $\cQ_\omega$ & $\tilde \cO \rbr*{\frac{\sigma^2(1+\omega) + \zeta^2 \omega}{\mu n \epsilon} + \frac{L (1+Z^2\omega/n)}{\mu}}$ & $\cO\rbr*{\frac{\sigma^2(1+\omega) + \zeta^2 \omega}{n \epsilon^2} + \frac{L(1+Z^2\omega/n)}{\epsilon} } \cdot R_0^2 $ & $\cO\rbr*{\frac{\sigma^2(1+\omega) + \zeta^2\omega}{n \epsilon^2} + \frac{L(1+Z^2\omega/n)}{\epsilon} } \cdot L F_0$ \\
D-EF-SGD &$\cC_\delta$ & $\tilde \cO \rbr*{ \frac{\sigma^2}{\mu n \epsilon} + \frac{\sqrt{L}(\sigma + \zeta/\sqrt{\delta})}{\mu \sqrt{\delta \epsilon} } + \frac{LZ)}{\mu \delta }}$ & $\cO \rbr*{ \frac{\sigma^2}{n \epsilon^2} + \frac{\sqrt{L}(\sigma + \zeta/\sqrt{\delta})}{\sqrt{\delta} \epsilon^{3/2} } + \frac{LZ}{\epsilon \delta }} \cdot R_0^2$ &  $\cO \rbr*{ \frac{\sigma^2}{n \epsilon^2} + \frac{\sigma + \zeta/\sqrt{\delta}}{\sqrt{\delta} \epsilon^{3/2} } + \frac{Z}{\epsilon \delta}} \cdot L F_0$\\
\midrule
DIANA\tnote{d} & $\cQ_\omega$ & $\tilde \cO \rbr*{\frac{\sigma^2(1+\omega)}{\mu n \epsilon} + \omega + \frac{L (1+\omega/n)}{\mu}}$ & $\cO\rbr*{\frac{\sigma^2(1+\omega) + (1+\omega)\omega \zeta_\star^2}{n \epsilon^2} + \frac{L(1+\omega/n)}{\epsilon} } \cdot R_0^2 $ & $\cO\rbr*{\frac{\sigma^2(1+\omega) + \zeta^2\omega }{n\epsilon^2} + \frac{L(1+Z^2\omega/n)}{\epsilon} } \cdot L F_0$ \\
D-EF b-corr\tnote{d,e} & $\cC_\delta$, $\cQ_\omega$ & $\tilde \cO \rbr*{ \frac{\sigma^2}{\mu n \epsilon} + \frac{\sqrt{L} \sigma}{\mu \sqrt{\delta \epsilon}} +  \frac{1+\omega}{\delta} + \frac{L}{\mu \delta} }$ & $\cO \rbr*{ \frac{\sigma^2}{n \epsilon^2} + \frac{\sqrt{L}(\sigma + \zeta_\star /\sqrt{\delta})}{\sqrt{\delta} \epsilon^{3/2} } + \frac{L}{\epsilon \delta }} \cdot R_0^2$ & $\cO \rbr*{ \frac{\sigma^2}{n\epsilon^2} + \frac{\sigma + \zeta/\sqrt{\delta}}{\sqrt{\delta} \epsilon^{3/2} } + \frac{Z}{\epsilon \delta}  } \cdot L F_0 $  \\
\midrule
D-QSGD & $\cQ_\omega$ linear & $\tilde \cO \rbr*{ \frac{\sigma^2(1+\omega)}{\mu n \epsilon} + \frac{L(1+\omega)}{\mu}}$ & $\cO \rbr*{ \frac{\sigma^2(1+\omega)}{n \epsilon^2} + \frac{L(1+\omega)}{\epsilon}} \cdot R_0^2$ &  $\cO \rbr*{\frac{\sigma^2(1+\omega)}{n} + \frac{1+\omega}{\epsilon}} \cdot L F_0$ \\
D-EF-SGD &$\cC_\delta$ linear& $\tilde \cO \rbr*{ \frac{\sigma^2}{\mu n \epsilon} + \frac{L}{\mu}}$ & $\cO \rbr*{ \frac{\sigma^2}{n \epsilon^2} + \frac{L}{\epsilon}} \cdot R_0^2$ &  $\cO \rbr*{\frac{\sigma^2}{n} + \frac{1}{\epsilon}} \cdot L F_0$ \\
\bottomrule
\end{tabular}
}
\begin{tablenotes}\footnotesize  %
    \parbox{\textwidth}{ \footnotesize 
     \item[a] Convergence $\Ef f(\xx_{\rm out})-f^\star \leq \epsilon$, where $\xx_{\rm out}$ is a random iterate, chosen with exponentially increasing probability in $t$.
     \item[b] Convergence $\Ef f(\xx_{\rm out})-f^\star \leq \epsilon$, where $\xx_{\rm out}$ is a uniformly at random chosen iterate.
     \item[c] Convergence $\Ef \norm{\nabla f(\xx_{\rm out})}^2 \leq \epsilon$, where $\xx_{\rm out}$ is a uniformly at random chosen iterate.
     \item[d] Require each $f_i$ to be $L$-smooth. For the first two columns require each $f_i$ to be convex. $\zeta^2 \geq \zeta_\star^2 :=\frac{1}{n}\sum_{i=1}^n \norm{\nabla f_i(\xx_\star)}^2$.
     \item[e] For the choice $\beta = \delta$.
     }
    \end{tablenotes}
\end{threeparttable}
\end{table}

\begin{figure*}[t]
\begin{minipage}[t]{1\linewidth}
\begin{minipage}[t]{0.48\linewidth}
\begin{algorithm}[H]
\caption{D-QSGD}
\label{alg:qsgd}
\begin{algorithmic}[1]
\myState{\textbf{Input:} $\xx_0$, $\gamma$, $\cQ_\omega$}
\For{$t=0,\dots,T-1$} \hfill $\triangledown$ worker side
 \myState{$\gg_t^i := \gg^i(\xx_t)$} \Comment{stochastic gradient}
 \myState{$\hat \Delta_t^i := \cQ_\omega(\gg_t^i)$}
   \vspace{2.3em}
 \myState{send to server: $\hat \Delta_t^i$} \hfill $\triangledown$ server side
 \myState{$\xx_{t+1}:= \xx_t - \frac{\gamma}{n} \sum_{i=1}^n \hat \Delta_t^i$}
\EndFor
\end{algorithmic}
\end{algorithm}
\end{minipage}
\begin{minipage}[t]{0.48\linewidth}
\begin{algorithm}[H]
\caption{D-EF-SGD}
\label{alg:ef}
\begin{algorithmic}[1]
\myState{\textbf{Input:} $\xx_0$, $\gamma$, $\cC_\delta$, $\ee_t^i=\0_d$}
\For{$t=0,\dots,T-1$} \hfill $\triangledown$ worker side
 \myState{$\gg_t^i := \gg^i(\xx_t)$} \Comment{stochastic gradient}
 \myState{$\hat \Delta_t^i := \cC_\delta(\ee_t^i + \gg_t^i)$}
 \myState{$\ee_{t+1}^i := \ee_t^i + \gg_t^i - \hat \Delta_t^i$} \vspace{0.7em}
 \myState{send to server: $\hat \Delta_t^i$} \hfill $\triangledown$ server side
 \myState{$\xx_{t+1}:= \xx_t - \frac{\gamma}{n} \sum_{i=1}^n  \hat \Delta_t^i$} 
\EndFor
\end{algorithmic}
\end{algorithm}
\end{minipage}

\begin{minipage}[t]{0.48\linewidth}
\begin{algorithm}[H]
\caption{DIANA}
\label{alg:diana}
\begin{algorithmic}[1]
\myState{\textbf{Input:} $\xx_0$, $\gamma$, $\cQ_\omega$, 
$\hh_0=\0_d$, $\hh_0^i=\0_d$, $\alpha \leq \frac{1}{1+\omega}$}
\For{$t=0,\dots,T-1$} \hfill $\triangledown$ worker side
 \myState{$\gg_t^i := \gg^i(\xx_t)$} \Comment{stochastic gradient}
 \myState{$\hat \Delta_t^i := \cQ_\omega(\gg_t^i-\hh_t^i)$}
 \myState{$\hh_{t+1}^i := \hh_t^i + \alpha \hat \Delta_t^i$}
 \vspace{1.7em}
 \myState{send to server: $\hat \Delta_t^i$} \hfill $\triangledown$ server side
 \myState{$\xx_{t+1}:= \xx_t - \gamma \hh_t - \frac{\gamma}{n} \sum_{i=1}^n  \hat \Delta_t^i$} 
 \myState{$\hh_{t+1} := \hh_t + \frac{\alpha}{n}\sum_{i=1}^n \hat \Delta_t^i$} \Comment{$\hh_t = \frac{1}{n} \sum_{i=1}^n \hh_t^i$}
\EndFor
\end{algorithmic}
\end{algorithm}
\end{minipage}
\begin{minipage}[t]{0.48\linewidth}
\begin{algorithm}[H]
\caption{D-EF-SGD with bias correction}
\label{alg:efnew}
\begin{algorithmic}[1]
\myState{\textbf{Input:} $\xx_0,\! \gamma,\! \cC_\delta,\! \cQ_\omega$, $\hh_0 \!=\! \hh_0^i \!=\! \ee_t^i \!= \! \0_d$, $\alpha \leq \frac{\beta}{1+\omega}$}
\For{$t=0,\dots,T-1$} \hfill $\triangledown$ worker side
 \myState{$\gg_t^i := \gg^i(\xx_t)$} \Comment{stochastic gradient}
 \myState{$\hat \Delta_t^i := \cC_\delta(\ee_t^i + \gg_t^i - \hh_t^i)$} \label{ln:diff1}
 \myState{$ \Delta_t^i :=\cQ_\omega(\gg_t^i-\hh_t^i)$} \label{ln:quant}
 \myState{$\ee_{t+1}^i := \ee_t^i + \gg_t^i - \hh_t^i - \hat \Delta_t^i$, $\hh_{t+1}^i := \hh_t^i + \alpha \Delta_t^i$} 
 \vspace{0.5em}
 \myState{send to server: $\hat \Delta_t^i$, $\Delta_t^i$} \hfill $\triangledown$ server side
 \myState{$\xx_{t+1}:= \xx_t - \gamma \hh_t -  \frac{\gamma}{n} \sum_{i=1}^n  \hat \Delta_t^i$} \label{ln:diff2}
 \myState{$\hh_{t+1}:= \hh_t + \frac{\alpha}{n} \sum_{i=1}^n  \Delta_t^i$}  \Comment{$\hh_t = \frac{1}{n} \sum_{i=1}^n \hh_t^i$}
\EndFor
\end{algorithmic}
\end{algorithm}
\end{minipage}

\end{minipage}
\caption{Algorithms with broadcast. All algorithms require coordination with a central parameter server, that broadcasts the updated parameter $\xx_t$ to the working nodes in each iteration. Quantization operators $\cQ_\omega$ are assumed to be independent of $t$ and $i$.}
\label{fig:1}
\end{figure*}

\section{Bias Correction for Improving Data-Depencence}
In this section, we discuss a technique proposed in~\cite{Mishchenko2019:diana} that allows to improve the algorithms dependence on the data-dissimilarity parameter for strongly convex problems. However, this technique requires slightly stronger assumptions, such as smoothness of each $f_i, i \in [n]$ and convexity.

\subsection{DIANA, Algorithm~\ref{alg:diana}}
\citet{Mishchenko2019:diana} introduced DIANA, an alternative to D-QSGD that allows to solve constrained optimization problems with quantized communication. 
Whilst this is one key applications of DIANA, we focus here on the benefits this method can offer for unconstrained optimization with communication compression.

A key mechanism in DIANA (Algorithm~\ref{alg:diana}) is that it maintains a sequence of auxiliary variables $\hh_t^i$ on each node $i \in [n]$, with the property $\hh_t^i \to \nabla f_i(\xx_\star)$ when $\xx_t \to \xx_\star$. These variables can be used to design compression operators with smaller variance: instead of compressing $\cQ_\omega(\gg_t^i)$ as D-QSGD, DIANA uses the quantizer $\hh_t^i + \cQ_\omega(\gg_t^i -\hh_t^i)$ instead in each round~\cite[see also][]{kuenstner2017:svrg}. This is still an unbiased quantizer, but the variance can decrease when $\hh_t^i$ is chosen in an optimal way (such observations were also stated in parallel work~\cite{Wangni2019:tng} but only rigorously proven in~\cite{Mishchenko2019:diana}).

Convergence rates for DIANA where first derived in~\cite{Mishchenko2019:diana} and later refined in~\cite{Horvath2019:vr}. None of these works presented convergence rates for just convex functions ($\mu=0$) and convergence rates for non-convex functions for arbitrary $\omega$-quantizers. The (small) improvement over \cite{Horvath2019:vr} stems from the fact that we consider an average of the iterates (and not the last one) as the output of the algorithm, and should be viewed as only a minor technical distinction.
\begin{theorem}[DIANA]
\label{thm:diana}
Let $f \colon \R^d \to \R$ be $\mu$-strongly convex and $L$-smooth and assume in addition that each $f_i \colon \R^d \to \R$ is $L$-smooth and convex.
Then there exists a stepsize $\gamma \leq \frac{1}{2L(1+2\omega /n)}$ such that for $\alpha = \frac{1}{1+\omega}$, after at most
\begin{align}
T = \tilde \cO \rbr*{\frac{\sigma^2(1+\omega)}{\mu n \epsilon} + \omega + \frac{L(1+\omega/n)}{ \mu}} \label{eq:bound_diana}
\end{align}
iterations of Algorithm~\ref{alg:diana} it holds $\Ef f(\xx_{\rm out}) - f^\star \leq \epsilon$, where $\xx_{\rm out} = \xx_t$ denotes an iterate $\xx_t \in \{\xx_0,\dots,\xx_{T-1}\}$, chosen at random with probability proportional to $\rbr*{1-\min \cbr*{\mu \gamma,\frac{\alpha}{2}}}^{-t}$.
\end{theorem}
\begin{remark}
The data-dissimilarity $\zeta^2$ appears only in poly-logarithmic factors in the convergence rate~\eqref{eq:bound_diana} and is thus hidden in the $\tilde{\cO}(\cdot)$ notation.
\end{remark}

\subsection{D-EF-SGD with bias correction, Algorithm~\ref{alg:efnew}}
Whilst DIANA is much less affected by non-iid data than D-QSGD, it still suffers from the linear slow-down in $(1+\omega)$ in the presence of stochastic noise. In this section we show that by applying error-feedback we can obtain a new algorithm with the optimal $\cO \rbr[\big]{\frac{\sigma^2}{\mu n \epsilon}}$ dependence on $\sigma^2$.

D-EF-SGD (Algorithm~\ref{alg:ef}) maintains local error correction terms $\ee_t^i$ on each node $i \in [n]$, however, $\ee_t^i \not \to \0_d$ in general, even when $\xx_t \to \xx_\star$. This causes the appearance of the $\zeta$ term in the rate. 
{\update
Following~\cite[see Appendix~\ref{app:efsgdbias}]{Kovalev2019:personal}, we study 
error-feedback with bias correction}, that is, D-EF-SGD with compressor $\cC_\delta(\ee_t^i + \gg_t^i -\hh_t^i)$ instead, where $\hh_t^i$ is chosen as to converge to $\nabla f_i(\xx_\star)$.
The scheme is stated in Algorithm~\ref{alg:efnew}.

\begin{theorem}[D-EF-SGD with bias correction]
\label{thm:ef-convex}
Let $f \colon \R^d \to \R$ be $\mu$-strongly convex and $L$-smooth and assume in addition that each $f_i \colon \R^d \to \R$ is $L$-smooth and convex and $\beta \leq 1$. Then there exists a stepsize $\gamma \leq \frac{\delta}{32L}$  such that after at most
\begin{align*}
 T= \tilde \cO \rbr*{\frac{\sigma^2}{\mu n \epsilon} + \rbr*{\frac{\sigma^2 L(1-\delta)}{\mu^2 \delta \epsilon}}^{1/2}  + 
\rbr*{\frac{\sigma^2 \beta L(1-\delta)}{\mu^2 \delta^2 \epsilon}}^{1/2}  
 + \frac{1+\omega}{\beta} + \frac{L}{\mu \delta} } 
\end{align*}
iterations of Algorithm~\ref{alg:efnew} it holds $\Ef f(\xx_{\rm out}) - f^\star \leq \epsilon$, where $\xx_{\rm out} = \xx_t$ denotes an iterate $\xx_t \in \{\xx_0,\dots,\xx_{T-1}\}$, chosen at random with probability proportional to $\rbr*{1-\min\cbr*{\frac{\gamma \mu}{2}, \frac{\alpha}{2}, \frac{\delta}{4}}}^{-t}$.
\end{theorem}
\begin{remark}
When $\sigma^2 > 0$ or when $(1+\omega)\leq\frac{L}{\mu}$, then the choice $\beta = \delta$ gives asymptotically the best complexity. When $\sigma^2=0$, choosing $\beta = 1$ gives the best linear convergence. In Table~\ref{tab:results} we list the result for the choice $\beta=\delta$, as we mostly focus on noisy stochastic problems in our discussion.
\end{remark}

\subsection{Discussion}
\paragraph{Linear convergence without stochastic noise.}
Without stochastic noise ($\sigma^2=0$), both algorithms presented in this section converge linearly on strongly-convex problems. 
For comparable choices of $\delta \approx \frac{1}{1 +\omega}$, and $Z=1$, the linear convergence rate of DIANA is better as the method can benefit from mini-batching effects. The speedup in $n$ in the $\tilde \cO \rbr[\big]{\frac{L(1+\omega/n}{\mu}}$ term stems from the fact that the quanitzation operators are independent on each node. In contrast, biased compressors cannot benefit from such effects and the $\tilde \cO \rbr[\big]{\frac{L}{\mu \delta}}$ term has  the best possible dependence on the compression parameter $\delta$ that cannot be improved in general~\cite[cf.][]{StichK19delays}. Both algorithms depend linearly on the condition number $\frac{L}{\mu}$. This dependence could be improved with acceleration techniques~\cite[cf.][]{lin15:catalyst}.

\paragraph{Dependence on data-dissimilarity.}
Whilst the convergence results on strongly-convex functions show that both DIANA and bias corrected D-EF-SGD only depend polylogarithmic on the data-dissimilarity parameter $\zeta_\star$, a closer inspection of the results in Table~\ref{tab:results} reveals that unfortunately both methods still depend on $\zeta$ without the convexity assumptions.

We conjecture that some partial improvements can obtained for non-convex problems, for instance by extending the analysis to non-convex problems with additional P\L{} condition. However, the current results seem to indicate that a fundamental different technique is required to remove the dependence on the data-dissimilarity parameter $\zeta$ from the convergence rates entirely.

\subsection{Convergence Proof for Bias Corrected D-EF-SGD}
Algorithm~\ref{alg:efnew} is a combination of D-EF-SGD with a feature of DIANA, and a convergence proof can be derived from techniques and tools developed in earlier work~\cite{StichK19delays,Horvath2019:vr}.
As a technical novelty, we here present a novel proof technique for general error-feedback SGD algorithms by introducing a Lyapunov function instead of the unrolling technique used in~\cite{StichK19delays}. Moreover, we also need a slight strengthening of one of the lemmas in~\cite{Horvath2019:vr} to show that the choice $\beta < 1$ gives an improvement in the convergence rate.

We give the convergence proof for the strongly convex case in the main text, all other proofs are given in the appendix. Define $X_t := \Ef \norm{\tilde \xx_t - \xx^\star}^2$. $F_t := \Ef f(\xx_t) - f^\star$. $E_t := \frac{1}{n} \sum_{i \in [n]} \Ef \norm{\ee_t^i}^2$, and $H_t := \frac{1}{n} \sum_{i \in [n]} \Ef \norm{\hh_t^i - \hh_\star^i}^2$ for $\hh_\star^i:=\nabla f_i(\xx_\star)$, $\hh_\star := \nabla f(\xx_\star)$ and the virtual sequence
\begin{align}
 \tilde \xx_0 &:= \xx_0\,, &  \tilde\xx_{t+1} &:= \tilde\xx_t - \frac{\gamma}{n} \sum_{i \in [n]} \gg_t^i \,.  \label{def:virtual}
\end{align}
We note:
\begin{align}
  \xx_{t+1} - \tilde \xx_{t+1} = \frac{\gamma}{n} \sum_{i \in [n]} \ee_{t+1}^i \,. \label{eq:error}
\end{align}

\paragraph{A decent lemma for convex functions.} 
First, we borrow a standard lemma for the analysis of error-feedback algorithms (for the proof see also Lemma~\ref{lemma:Xn} in the appendix).

\begin{lemma}[{\citet[Lemma~7]{StichK19delays}}]\label{lemma:X}
Let $f$ be $L$-smooth and $\mu$-convex. If the stepsize $\gamma \leq \frac{1}{4L}$, then it holds for the iterates of Algorithms~\ref{alg:ef} and~\ref{alg:efnew}:
\begin{align}
  X_{t+1} &\leq \left(1-\frac{\gamma \mu}{2} \right)  X_t - \frac{\gamma}{2}  F_t +  \frac{\gamma^2\sigma^2}{n} + 3L\gamma^3  E_t \,. \label{eq:X}
\end{align}
\end{lemma}

\paragraph{Bound on the error.} Next, we derive a recursive bound on $E_t$.
\begin{lemma}\label{lemma:E}
 It holds
\begin{align}
  E_{t+1} \leq \rbr*{1-\frac{\delta}{2}}  E_t +  \frac{4(1-\delta)}{\delta} \rbr*{2 L F_t +  H_t}    + (1-\delta) \sigma^2 \,. \label{eq:E}
\end{align}
\end{lemma}
\begin{proof}
By using the definition $\ee_{t+1}^i = \ee_t^i + \gg_t^i - \hh_t^i - \hat \Delta_t^i$, we obtain:
\begin{align}
 \E_{\bxi_t^i,\cC_\delta} \norm{\ee_{t+1}^i}^2 &= \E_{\bxi_t^i,\cC_\delta} \norm{\ee_t^i + \gg_t^i - \hh_t^i - \hat \Delta_t^i}^2 \notag \\
 & \stackrel{\eqref{def:compressor}}{\leq} (1-\delta) \E_{\bxi_t^i} \norm{\ee_{t}^i + \gg_t^i - \hh_t^i}^2 \notag \\ 
 &\stackrel{\eqref{def:noise}}{=} (1-\delta) \E_{\bxi_t} \norm{\ee_{t}^i + \nabla f_i(\xx_t) + \bxi_t^i - \hh_t^i}^2 \notag  \\ &\stackrel{\eqref{eq:biasvariance}}{=} (1-\delta) \norm{\ee_{t}^i + \nabla f_i(\xx_t) - \hh_t^i}^2 + (1-\delta)\E_{\bxi_t^i} \norm{\bxi_t^i}^2 \notag \\
 &\stackrel{\eqref{def:noise}}{\leq } (1-\delta) \norm{\ee_{t}^i + \nabla f_i(\xx_t) - \hh_t^i}^2 + (1-\delta)\sigma^2  \notag \\
 &\stackrel{\eqref{eq:trick}}{\leq} (1-\delta/2) \norm{\ee_t^i}^2 + \frac{2(1-\delta)}{\delta}\norm{\nabla f_i(\xx_t) - \hh_t^i}^2 + (1-\delta)\sigma^2\,. \label{eq:Econtinue}
\end{align}
Using smoothness (and convexity) of $f_i(\xx)$, we observe
\begin{align*}
 \norm{\nabla f_i(\xx_t) - \hh_t^i}^2 &\stackrel{\eqref{eq:young}}{\leq} 2\norm{\nabla f_i(\xx_t) - \hh_\star^i}^2 + 2\norm{\hh_t^i - \hh_\star^i}^2 \\
 &\stackrel{\eqref{eq:smoothL}}{\leq} 4L \rbr*{ f_i(\xx_t) - f_i(\xx^\star) + \lin{\nabla f_i(\xx^\star),\xx_t -\xx^\star} } + 2\norm{\hh_t^i - \hh_\star^i}^2
\end{align*}
The claim now follows by summing and averaging over $i \in [n]$.
\end{proof}

\paragraph{Estimate $H$.}
The next lemma tightens \cite[Lemma 2]{Horvath2019:vr} (with $\alpha^2$ instead of only $\alpha$ in the last term).
\begin{lemma}\label{lemma:Hunbiased}
Let $\hh_t^i$ be updated with an unbiased quantizer $\cQ_\omega$, $\alpha \leq \frac{1}{\omega + 1}$, and stepsize $\alpha$. Then
\begin{align}
  H_{t+1} \leq (1-\alpha ) H_t + 2 \alpha  L  F_t + \alpha^2 (1+\omega) \sigma^2\,. \label{eq:Hunbiased}
\end{align}
\end{lemma}
\begin{proof}
Closely following~\cite{Horvath2019:vr} we observe
\begin{align}
 \E_{\bxi_t^i, \cQ_\omega} \norm{\hh_{t+1}^i - \hh_\star^i}^2 
 &= \norm{\hh_t^i - \hh_\star^i}^2 + 2\alpha \lin{ \E_{\bxi_t^i, \cQ_\omega}  \Delta_t^i, \hh_t^i -\hh_\star^i} + \alpha^2 \E_{\bxi_t^i, \cQ_\omega} \norm{ \Delta_t^i}^2 \notag \\
 &\stackrel{\eqref{def:quantizer}}{\leq} \norm{\hh_t^i - \hh_\star^i}^2 + 2\alpha \lin{\nabla f_i(\xx_t)-\hh_t^i, \hh_t^i -\hh_\star^i} + \alpha^2 (1+\omega) \E_{\bxi_t^i} \norm{\gg_t^i - \hh_t^i}^2  \notag \\
 \begin{split}
 &\stackrel{\eqref{eq:biasvariance}}{=} \norm{\hh_t^i - \hh_\star^i}^2 + 2\alpha \lin{\nabla f_i(\xx_t)-\hh_t^i, \hh_t^i -\hh_\star^i} \\ &\qquad + \alpha^2 (1+\omega) \rbr*{\norm{\nabla f(\xx_t) -\hh_t^i}^2 + \E_{\bxi_t^i}\norm{\bxi_t^i}}  
 \end{split}  \notag \\
 &\stackrel{\eqref{def:noise}}{\leq} \norm{\hh_t^i - \hh_\star^i}^2 + 2\alpha \lin{\nabla f_i(\xx_t)-\hh_t^i, \hh_t^i -\hh_\star^i} + \alpha \norm{\nabla f(\xx_t)- \hh_t^i}^2 + \alpha^2 (1+\omega) \sigma^2  \notag \\
 &=(1-\alpha)\norm{\hh_t^i - \hh_\star^i}^2 + \alpha \norm{\nabla f_i(\xx_t) -\hh_\star^i}^2 + \alpha^2 (1+\omega) \sigma^2 \label{eq:Hcontinue}   
\end{align}
where we used the equality $2\lin{\aa,\bb} + \norm{\bb}^2 = \norm{\aa + \bb}^2 - \norm{\aa}^2$ for vectors $\aa,\bb \in \R^d$ for the last estimate.
The claim follows with~\eqref{eq:g-h}.
\end{proof}

\noindent We can summarize the statements of Lemmas~\ref{lemma:X}--\ref{lemma:Hunbiased} in the following descent lemma.
\begin{lemma}[Lyapunov function]
\label{lemma:convex}
Let $f$ be $L$-smooth, $\mu$-convex and each $f_i \colon \R^d \to \R$ convex and $L$-smooth, the stepsize $\gamma \leq \frac{\delta}{34L}$ and $\alpha = \frac{\beta}{1+\omega}$ with a parameter $\beta \leq 1$. Then
\begin{align}
 \Psi_{t+1} \leq (1-c)\Psi_t - \gamma \frac{F_t}{4} + \gamma^2 \frac{ \sigma^2}{n} + \gamma^3 \frac{4 L(1-\delta)\sigma^2}{\delta} + \gamma^3 \frac{32 \beta L (1-\delta)\sigma^2}{\delta^2} \,,
\end{align}
for $\Psi_t := X_t + a E_t + b H_t$ with $a = \frac{12 \gamma^3 L}{\delta}$ and $b = \frac{8a (1-\delta)}{\alpha \delta}$ and $c= \min\cbr*{\frac{\gamma \mu}{2}, \frac{\alpha}{2}, \frac{\delta}{4}}$.
\end{lemma}

\begin{proof}
Observe that it holds
$\rbr*{1 -  \frac{\delta}{2} + \frac{3\gamma^3 L}{a}} \leq \rbr*{1-\frac{\delta}{4}}$ and
$\rbr*{1-\alpha + \frac{4a(1-\delta)}{b \delta}} \leq \rbr*{1-\frac{\alpha}{2}}$
 by the choice of $a,b$.
 Therefore
 \begin{align*}
  \Psi_{t+1} &= X_{t+1} + a E_{t+1} + b H_{t+1} \\
  &\stackrel{\mathclap{\eqref{eq:X},\eqref{eq:E},\eqref{eq:Hunbiased}}}{\leq}  \;\;\;\;\;\; \left(1-\frac{\gamma \mu}{2} \right) X_{t} + a\rbr*{1 - \frac{\delta}{2} +  \frac{3 \gamma^3 L}{a}} E_t + b\rbr*{1-\alpha + \frac{4a(1-\delta)}{b \delta}} H_t \\
  &\qquad\quad   + \rbr*{ \frac{8a(1-\delta)L}{\delta} + 2\alpha b L   -\frac{\gamma}{2} }F_t + \rbr*{\frac{\gamma^2}{n} + a(1-\delta) + \alpha^2 b(1+\omega) }\sigma^2  \\
  &\leq (1-c)\Psi_t - \gamma \frac{F_t}{4} + \gamma^2 \frac{\sigma^2}{n} + \gamma^3 \frac{12 L(1-\delta)\sigma^2}{\delta} + \gamma^3 \frac{96 \beta L (1-\delta)\sigma^2}{\delta^2}\,,
 \end{align*}
 where we used the choice of the parameters. For the $F_t$ terms:
 \begin{align*}
 \frac{8a(1-\delta)L}{\delta} + 2\alpha b L   -\frac{\gamma}{2} = \frac{96 \gamma^3 L^2 (1-\delta)}{\delta^2} + \frac{192  L^2 \gamma^3 (1-\delta)}{\delta^2} - \frac{\gamma}{2} \leq -\frac{\gamma}{4}\,,
 \end{align*}
 for $\gamma \leq \frac{\delta}{34L}$, $\alpha \leq \frac{1}{1+\omega}$. And for the $\sigma^2$ term, with $\alpha = \frac{\beta}{1+\omega}$,
\begin{align*}
 \frac{\gamma^2}{n} + a(1-\delta) + \alpha^2 b(1+\omega) &= \frac{\gamma^2}{n} + \frac{12\gamma^3 L(1-\delta)}{\delta} + \frac{96 \beta \gamma^3 L(1-\delta)}{\delta^2}\,.  \qedhere
\end{align*}
\end{proof}

\begin{proof}[Proof of Theorem~\ref{thm:ef-convex}]
Lemma~\ref{lemma:convex}, together with Lemma~\ref{lemma:convergence-strong-convex} 
and Remark~\ref{rem:convergence-strong-convex} 
show  the claim.
\end{proof}

\section{Avoiding Data-Dependent Rates with Linear Compressors}
Before concluding this note, we like to remark that with a very simple modification the data dependent parameter $\zeta^2$ can entirely be removed from the convergence rates in D-QSGD and D-EF-SGD. This is possible while leaving the algorithms unchanged, but instead we propose to restrict the class of admissible quantization (or compression) operators.

\begin{figure*}[t]
\begin{minipage}[t]{1\linewidth}
\begin{minipage}[t]{0.48\linewidth}
\begin{algorithm}[H]
\caption{D-QSGD (synchronized linear quantizers)}
\label{alg:qsgd_all}
\begin{algorithmic}[1]
\myState{\textbf{Input:} $\xx_0$, $\gamma$}
\myState{\textbf{Input:} $(\cQ_t)_{t=0}^{T-1}$ linear $\cQ_\omega$ quantizers}
\For{$t=0,\dots,T-1$} \hfill $\triangledown$ worker side
 \myState{$\gg_t^i := \gg^i(\xx_t)$} \Comment{stochastic gradient}
 \myState{$\hat \Delta_t^i := \cQ_t(\gg_t^i)$}
   \vspace{2.1em}
 \myState{all-reduce: $\hat \Delta_t^i$} 
 \myState{$\xx_{t+1}:= \xx_t - \frac{\gamma}{n} \sum_{i=1}^n \hat \Delta_t^i$}
\EndFor
\end{algorithmic}
\end{algorithm}
\end{minipage}
\begin{minipage}[t]{0.48\linewidth}
\begin{algorithm}[H]
\caption{D-EF-SGD (synchronized linear compressors)}
\label{alg:ef_all}
\begin{algorithmic}[1]
\myState{\textbf{Input:} $\xx_0$, $\gamma$, $\ee_t^i=\0_d$}
\myState{\textbf{Input:} $(\cC_t)_{t=0}^{T-1}$ linear $\cC_\delta$ compressors}
\For{$t=0,\dots,T-1$} \hfill $\triangledown$ worker side
 \myState{$\gg_t^i := \gg^i(\xx_t)$} \Comment{stochastic gradient}
 \myState{$\hat \Delta_t^i := \cC_t(\ee_t^i + \gg_t^i)$}
 \myState{$\ee_{t+1}^i := \ee_t^i + \gg_t^i - \hat \Delta_t^i$} \vspace{0.5em}
 \myState{all-reduce: $\hat \Delta_t^i$} 
 \myState{$\xx_{t+1}:= \xx_t - \frac{\gamma}{n} \sum_{i=1}^n  \hat \Delta_t^i$} 
\EndFor
\end{algorithmic}
\end{algorithm}
\end{minipage}

\end{minipage}
\caption{Algorithms using synchronized linear compressors (otherwise identical to D-QSGD and D-EF-SGD in Algorithms~\ref{alg:qsgd}--\ref{alg:ef}). 
These algorithms can either be implemented with a parameter server and broadcast (as in Figure~\ref{fig:1}), or with all-reduce as shown here.
}
\end{figure*}

The main component in the convergence proofs was to estimate the variance (for D-QSGD) and the bound on the memory $E_t$ (for D-EF-SGD).
For certain classes of compressors these bounds can significantly be improved. As one example, we here highlight linear compressors for which it holds $\E_{\cQ_\omega} \norm*{\frac{1}{n}\sum_{i=1}^n \cQ_{\omega} ( \gg_t^i )}^2 \leq (1+\omega) \norm{\frac{1}{n}\sum_{i=1}^n \gg_t^i}^2 $. With this property it is immediate to see that the proof of D-QSGD boils down to the $n=1$ worker case, and the data-dependent terms disappear (similarly for error-feedback algorithms with compressors). 

For example, consider linear sketching operators $\cS_{\mV_t}$, defined in~\eqref{def:S}, with a sketching matrix $\mV_t$ that can change over iterations $t$, but is \emph{identical} on all $n$ nodes at every $t$. Then it holds
\begin{align*}
 \norm[\Big]{\frac{1}{n}\sum_{i \in [n]} \cS_{\mV_t}(\gg_t^i - \ee_t^i)}^2 = \norm*{\cS_{\mV_t}(\gg_t - \ee_t)}^2 \leq (1-\delta) \norm{\gg_t - \ee_t}^2\,,
\end{align*}
where here $\ee_t := \frac{1}{n}\sum_{i \in [n]} \ee_t^i$. An analogous observation holds for rescaled (unbiased) sketching operators. We summarize the consequences of this observation in Table~\ref{tab:results}.

The benefits given by linear compressors have been exploited in some recent works~\cite[such as][]{Rothchild2020:fed,Vogels2020:power}. We believe---given the benefits of the much improved convergence rates and possibility to use efficient all-reduce implementations---the small overhead of synchronizing the compressors can be beneficial in many practical settings, especially for distributed optimization in data-centers. For instance, (pseudo-)random projections can be implemented with the help of a shared random seed without overhead, and certain data-adaptive protocols can also be implemented without a central coordinator. However, for optimization in federated learning scenarios, where communication is extremely limited and all-reduce not available, or when the data distribution is very different on each node, then the optimal trade-off between linear and locally adaptive compressors still remains to be studied in detail.

\section{Conclusion}
In this work we derive new and improved converge rates for D-QSGD and D-EF-SGD. 
Our derivations reveal that both methods can suffer a slow-down in the case of heavily skewed data-distributions on the nodes. Whilst this slow-down can be linear in the data-dissimilarity parameter for D-QSGD, it is much less severe for D-EF-SGD, where the data-distribution does not impact the asymptotically dominating terms in the convergence rate. We further present a new {\update analysis for a} bias corrected variant of D-EF-SGD that is even more mildly  affected by data-skewness on strongly convex problems (similar to DIANA, while maintaining the optimal stochastic terms as in vanilla D-EF-SGD).
Furthermore, we point out that when using linear compressors, this slow-down can entirely be avoided for all considered classes of smooth optimization problems. Whilst this small fix might be an interesting avenue for practical applications, it remains an open theoretical problem data-dependence of the convergence rates can be achieved for general compressors and problem classes.

\section*{Acknowledgments}
We thank Anastasia Koloskova, Frederik K\"{u}nstner and Martin Jaggi for discussions and comments on this manuscript.

\bibliography{diana-ef}
\bibliographystyle{myplainnat}

\newpage
\onecolumn
\appendix

\section{Technical Lemmas}
We list a few technical lemmas that are helpful in the proofs below:
\begin{itemize}
 \item For a random variable $X$:
\begin{align}
\E \norm{X - \E X}^2 = \E \norm{X}^2 - \norm{\E X}^2 \label{eq:biasvariance}
\end{align}
 \item For pairwise independent random variables $X_1,\dots,X_k$:
 \begin{align}
   \E \norm[\Big]{\sum_{i \in [k]} X_i - \E X_i}^2 = \sum_{i \in [k]} \E \norm{X_i - \E X_i}^2 \label{eq:independent}
 \end{align}
 \item In contrast, for any arbitrary $k$ vectors $\aa_1,\dots,\aa_k\in \R^d$:
\begin{align}
 \norm[\Big]{\sum_{i \in [k]} \aa_i}^2 \leq k \sum_{i \in [k]} \norm{\aa_i}^2 \label{eq:norm}
\end{align}
 \item For any vectors $\aa,\bb \in \R^d$ and $\eta > 0$:
 \begin{align}
 \norm{\aa + \bb}^2 \leq (1+\eta) \norm{\aa}^2 + (1+1/\eta) \norm{\bb}^2 \label{eq:young}
\end{align}
\item As a consequence, we will be often using the inequality
\begin{align}
 (1-\delta)  \norm{\aa + \bb}^2  \leq (1-\delta/2) \norm{\aa}^2 + \frac{2(1-\delta)}{\delta} \norm{\bb}^2 \label{eq:trick}
\end{align}
for $\delta \in (0,1]$. This follows from~\eqref{eq:young} with $\eta = \frac{\delta}{2(1-\delta)}$.
\item For $L$-smooth and convex functions we have the inequality:
\begin{align}
 \frac{1}{2L}\norm{\nabla f(\xx) - \nabla f(\yy)}^2 \leq f(\yy) - f(\xx) + \lin{\nabla f(\xx), \yy- \xx}\,, & &\forall \xx,\yy \in \R^d \label{eq:smoothL}
\end{align}
\item It is also useful to note: for convex and smooth $f_i$, with $\hh_\star^i:=\nabla f(\xx_\star)$, $\hh_\star = \nabla f(\xx_\star)$:
\begin{align}
\frac{1}{n}\sum_{i \in [n]} \E_{\bxi_t} \norm{\gg_t^i- \hh_\star^i}^2 
&\stackrel{\eqref{def:noise}}{\leq} \frac{1}{n}\sum_{i \in [n]} \norm{\nabla f_i(\xx_t)- \hh_\star^i}^2 + \sigma^2 \notag \\
&\stackrel{\eqref{eq:smoothL}}{\leq} \frac{2L}{n} \sum_{i \in [n]} \rbr[\big]{ f_i(\xx_t) - f_i(\xx^\star) + \lin{\nabla f_i(\xx^\star),\xx_t -\xx^\star} } + \sigma^2 \notag \\
&= 2L \rbr[\big]{f(\xx_t) - f^\star} + \sigma^2 \label{eq:g-h}
\end{align}
\end{itemize}

\section{D-QSGD}
In this section we prove Theorem~\ref{thm:qsgd}. The iterations of D-QSGD can be written as
\begin{align*}
 \xx_{t+1} &:= \xx_t - \gamma \hat \gg_t\,, &
 &\text{where} &
 \hat \gg_t &=\frac{1}{n} \sum_{i=1}^n \cQ_\omega(\gg_t^i(\xx_t))\,,
\end{align*}
and $\cQ_\omega$ are (independent) $\omega$-quantizers and $\gg_t^i$, $i \in [n]$ are unbiased gradient estimators $\gg_t^i(\xx):=\nabla f_i(\xx_t) + \bxi_t^i$ on each worker $i \in [n]$. With the observation that the update in each iteration is an unbiased estimator of the gradient, $\E_{\bxi,\cQ_\omega} \hat \gg_t = \nabla f(\xx)$, the convergence proof follows directly from standard SGD analyses with the proper upper bound of the variance of the $\hat \gg_t$ estimator.

\begin{lemma}[Variance of D-QSGD update]
\label{lemma:qsgdvariance}
For $\hat \gg(\xx) := \frac{1}{n} \sum_{i=1}^n \cQ_\omega(\gg^i(\xx))$ with independent $\omega$-quantizers $\cQ_\omega$ and unbiased gradient estimators with $\sigma^2$-bounded variance, it holds $\E_{\bxi,\cQ_\omega} \hat \gg=\nabla f(\xx)$, $\forall \xx \in \R^d$ and 
\begin{align}
 \E_{\bxi,\cQ_\omega} \norm*{\hat \gg(\xx) -\nabla f(\xx)}^2 &\leq \frac{\omega Z^2}{n} \norm{\nabla f(\xx)}^2 + \frac{\omega \zeta^2}{n} + \frac{\sigma^2(1+\omega)}{n} \,, & \forall \xx \in \R^d\,. \label{eq:qsgdvariance}
\end{align}
\end{lemma}
\begin{proof}
We derive:
\begin{align*}
 \E_{\bxi,\cQ_\omega} \norm{ \hat \gg (\xx)- \nabla f(\xx) }^2
 &= \E_{\bxi,\cQ_\omega} \norm*{ \frac{1}{n}\sum_{i=1}^n \rbr*{ Q(\gg^i(\xx))-\nabla f_i(\xx) } }^2 \\ 
 &\stackrel{\eqref{eq:independent}}{=} \frac{1}{n^2} \sum_{i=1}^n \E_{\bxi^i,\cQ_\omega} \norm{Q(\gg^i(\xx))- \nabla f_i(\xx)}^2 \\
 &\stackrel{\eqref{eq:biasvariance}}{=} \frac{1}{n^2} \sum_{i=1}^n \rbr*{\E_{\bxi^i,\cQ_\omega} \norm{Q(\gg^i(\xx))}^2 - \norm{\nabla f_i(\xx)}^2} \\
 &\stackrel{\eqref{def:quantizer}}{\leq} \frac{1}{n^2}\sum_{i=1}^n \rbr*{(1+\omega)\E_{\bxi^i}^2 \norm{\gg^i(\xx)}^2 - \norm{\nabla f_i(\xx)}^2 }
 \\ 
 & \stackrel{\eqref{def:noise}}{\leq} \frac{\omega}{n^2} \sum_{i=1}^n \norm{\nabla f_i(\xx)}^2 + \frac{\sigma^2 (1+\omega)}{n} \\
 &\stackrel{\eqref{def:zeta}}{\leq} \frac{\omega Z^2}{n} \norm{\nabla f(\xx)}^2 + \frac{\omega \zeta^2}{n}+ \frac{\sigma^2 (1+\omega)}{n} \qedhere
\end{align*}
\end{proof}
\begin{lemma}[Decent Lemma for D-QSGD]
\label{lemma:qsgd_descent}
For iterates $\xx_t$ defined as in QSGD, and $\gamma \leq \frac{1}{2L(1+Z^2\omega/n)}$ it holds for $L$-smooth functions
\begin{align}
F_{t+1} \leq F_t - \gamma   \frac{ G_t}{2} + L\gamma^2 \frac{\sigma^2(1+\omega) + \zeta^2 \omega}{2n}\,,
\end{align}
and if the function is in addition $\mu$-convex:
\begin{align}
 X_{t+1} \leq (1-\mu \gamma) X_t - \gamma F_t + \gamma^2 \frac{\sigma^2(1+\omega) + \zeta^2 \omega}{n}\,,
\end{align}
with $X_t := \Ef \norm{\xx_t - \xx_\star}^2$, $F_t = \Ef f(\xx_t) - f^\star$, $G_t := \Ef \norm{\nabla f(\xx_t)}^2$.
\end{lemma}
\begin{proof}
 For convex functions, it holds
 \begin{align*}
  X_{t+1} &= X_t - 2\gamma\lin{\E_{\bxi} \hat \gg_t, \xx_t -\xx_\star}^2 + \gamma^2 \E_{\bxi} \norm{\hat \gg_t}^2 \\
   &\stackrel{\mathclap{\eqref{def:convex},\eqref{eq:biasvariance}}}{\leq}\;\; (1-\mu \gamma) X_t - 2 \gamma F_t + \gamma^2 \rbr*{ \E_{\bxi}\norm{\hat \gg_t - \nabla f(\xx_t)}^2 + \norm{\nabla f(\xx_t)}^2} \\
  & \stackrel{\eqref{eq:qsgdvariance}}{\leq} (1-\mu \gamma) X_t - 2 \gamma F_t + \gamma^2 \rbr*{1+\frac{Z^2 \omega}{n}} \norm{\nabla f(\xx_t)}^2 + \gamma^2 \frac{\sigma^2(1+\omega) + \zeta^2 \omega}{n}\,,
 \end{align*}
 and the claim follows with $\E \norm{\nabla f(\xx_t)}^2 \stackrel{\eqref{eq:smoothL}}{\leq} 2LF_t$ and the choice of $\gamma$. For smooth functions,
 \begin{align*}
 F_{t+1} \leq F_t - \gamma G_t + \frac{L\gamma^2}{2} \rbr*{ \E \norm{\hat \gg(\xx_t) - \nabla f(\xx_t)}^2 + \norm{\nabla f(\xx_t)}^2}
 \end{align*}
 and again Lemma~\ref{lemma:qsgdvariance} together with $\gamma \leq \frac{1}{2L(1+Z^2\omega/n)}$ show the claim.
\end{proof}
Now the convergence proof for strongly convex functions follows with Lemma~\ref{lemma:convergence-strong-convex} and with Lemma~\ref{lemma:convergence} for convex and only smooth functions.

\section{DIANA}
In this section we prove Theorem~\ref{thm:diana}. For the strongly convex case this proof this is a direct copy of the template from~\cite{Mishchenko2019:diana,Horvath2019:vr}.
For the proof, it will be useful to define
\begin{align*}
 \hat \gg_t &:= \frac{1}{n}  \sum_{i \in [n]} (\hh_t^i + \hat \Delta_t^i)\,, &
  \gg_t := \frac{1}{n}  \sum_{i \in [n]} \gg_t^i \,,
\end{align*}
and to observe $\E_{\bxi_t,\cQ_\omega} \hat \gg_t = \nabla f(\xx_t)$.
We need a few observations:
\begin{lemma}[{\citet[Lemma 1 \& 2]{Horvath2019:vr}}]
\label{lemma:diana}
Let each $f_i$ be convex. Then,
for $\hh_\star = \nabla f(\xx^\star)$, $\hh_\star^i := \nabla f_i(\xx_\star)$, $F_t = \Ef f(\xx_t)-f^\star$ and 
$H_t = \frac{1}{n} \sum_{i \in [n]} \Ef \norm{\hh_t^i - \hh_t^\star}^2$,
\begin{align}
 \Ef \norm{\hat \gg_t - \hh_\star}^2 \leq 2L \rbr*{1+\frac{2\omega}{n}} F_t + \frac{(1+\omega)\sigma^2}{n} + \frac{2\omega}{n} H_t\,, \label{diana:2}
\end{align}
and for $\alpha \leq \frac{1}{1+\omega}$ 
 \begin{align}
  H_{t+1} \leq (1-\alpha) H_t + \alpha \rbr*{2 L F_t  + \sigma^2} \,. \label{eq:diana:3}
 \end{align}
\end{lemma}

\noindent Without convexity assumption, we can only derive the following weaker statements:
\begin{lemma}[Without convexity]
\label{lemma:Hnonconvex}
 It holds
 \begin{align*}
  \Ef \norm{\hat \gg_t - \nabla f(\xx_t)}^2 & \leq \frac{\omega Z^2}{n} \norm{\nabla f(\xx)}^2 + \frac{\omega \zeta^2}{n}+ \frac{\sigma^2 (1+\omega)}{n} + \frac{\omega}{n} H_t \,.
 \end{align*}
 For $\alpha \leq 1$ 
 and $H'_t := \frac{1}{n}\sum_{i \in [n]} \Ef \norm{\hh_t^i}^2$ it holds
 \begin{align*}
 H'_{t+1} \leq (1-\alpha) H'_t + \alpha \rbr*{ \zeta^2 +  Z^2 G_t } + \alpha^2 \sigma^2 (1+\omega) \,.
 \end{align*}
\end{lemma}
\begin{proof}
The proof of the first claim follows along the same lines as the proof of Lemma~\ref{lemma:qsgdvariance}, for the proof of the second claim we refer to the comments below Lemma~\ref{lemma:Enonconvex}.
\end{proof}

\noindent Next, we state a decent lemma. In contrast to \cite{Mishchenko2019:diana,Horvath2019:vr} that considered more general proximal updates, we consider here unconstrained optimization and present a simplified result.
\begin{lemma}[Lyapunov function]\label{lemma:diana2}
Let $f$ be $L$-smooth and the stepsize $\gamma \leq \frac{1}{2L(1+2Z^2\omega/n)}$ and  $\alpha = \frac{1}{1+\omega}$. Then
 \begin{align*}
  \Xi_{t+1} \leq \Xi_{t} - \frac{\gamma}{4} G_t + L\gamma^2 \frac{\sigma^2(2+\omega) + 2\zeta^2 \omega}{2n} \,,
 \end{align*}
 for $b=\frac{\gamma^2 L \omega}{2\alpha n}$ and $\Xi_t := F_t + b H'_t$, with $F_t = \Ef f(\xx_t)-f^\star$,  $H'_t := \frac{1}{n}\sum_{i \in [n]} \Ef \norm{\hh_t^i}^2$.

\noindent If in addition $f$ $\mu$-strongly convex and each $f_i$ convex and $L$-smooth, and stepsize $\gamma \leq \frac{1}{2L(1+8\omega/n)}$. Then it holds
\begin{align*}
  \Psi_{t+1} \leq \rbr*{1-\min \cbr*{\mu \gamma,  \frac{\alpha}{2} } } \Psi_t - \frac{\gamma}{2} F_t + \frac{5 \gamma^2 (1+\omega) \sigma^2}{n} \,,
\end{align*}
for $\Psi_t := X_t + a H_t$, with
$a= \frac{4 \gamma^2 \omega}{\alpha n}$,
$X_t = \Ef \norm{\xx_t - \xx_\star}^2$ and
$F_t = \Ef f(\xx_t)-f^\star$ and 
$H_t = \frac{1}{n} \sum_{i \in [n]} \Ef \norm{\hh_t^i - \hh_t^\star}^2$ as before.
\end{lemma}

\begin{proof}
We follow the usual template, and start with the convex case:
\begin{align*}
 \E_{\bxi_t, \cQ_\omega} \norm{\xx_{t+1} - \xx^\star}^2 
 &= \norm{\xx_t - \xx^\star}^2 - 2\gamma \lin{\E_{\bxi_t, \cQ_\omega} \hat \gg_t, \xx_t - \xx^\star} + \gamma^2 \E_{\bxi_t, \cQ_\omega} \norm{\hat \gg_t}^2 \\
 &=
 \norm{\xx_t - \xx^\star}^2 - 2\gamma \lin{\nabla f(\xx_t), \xx_t - \xx^\star} + \gamma^2 \E_{\bxi_t, \cQ_\omega} \norm{\hat \gg_t - \hh_\star}^2 \\
 &\stackrel{\eqref{def:convex}}{\leq} (1-\mu \gamma)  \norm{\xx_t - \xx^\star}^2 - 2\gamma (f(\xx_t) - f^\star)  + \gamma^2 \E_{\bxi_t, \cQ_\omega} \norm{\hat \gg_t - \hh_\star}^2 \,.
\end{align*}
By Lemma~\ref{lemma:diana}, and taking full expectation
\begin{align}
 X_{t+1} &\stackrel{\eqref{diana:2}}{\leq} (1-\mu \gamma) X_t -  \gamma \rbr*{2L\gamma \rbr*{1+\frac{2\omega}{n}} - 2} F_t + \frac{\gamma^2 (1+\omega)\sigma^2}{n} + \frac{2 \gamma^2 \omega}{n} H_t  \notag \\
 &\leq (1-\mu \gamma) X_t -  \gamma F_t + \frac{\gamma^2 (1+\omega)\sigma^2}{n} + \frac{2 \gamma^2 \omega}{n} H_t \,, \label{eq:3335}
\end{align}
with the choice of the stepsize, $\gamma \leq \frac{1}{2L(1+2\omega/n)}$.

\noindent We now combine the bound in~\eqref{eq:3335} with the estimate on $H_t$ provided in Lemma~\ref{lemma:diana}.
With the observation that for the chosen $a=\frac{4\gamma^2 \omega }{\alpha n}$ it holds $\rbr*{1+\frac{2 \gamma^2 \omega}{a n}}(1-\alpha) \leq (1-\alpha/2)$ and we obtain
\begin{align}
 \Psi_{t+1} &= X_{t+1} + a H_{t+1} \notag \\
 &\stackrel{\mathclap{\eqref{eq:Hunbiased}, \eqref{eq:3335}}}{\leq}\;\; (1-\mu \gamma)X_t + a \rbr*{1+\frac{2 \gamma^2 \omega}{a n}}(1-\alpha) H_t + \rbr*{2a \alpha L  -\gamma }F_t + \rbr*{\frac{\gamma^2}{n} + a \alpha^2}  \sigma^2  (1+\omega) \notag \\
 &\leq (1-c) \Psi_t + \rbr*{2a \alpha L  -\gamma }F_t + \rbr*{\frac{\gamma^2}{n} + a \alpha^2}  \sigma^2  (1+\omega) \,, \notag
\end{align}
for $c=\min \cbr*{\mu \gamma,  \frac{\alpha}{2} }$.
By the choice of $a$, we can simplify the terms in the two brackets:
\begin{align*}
2a \alpha L  -\gamma = \frac{8 \gamma^2 L \omega}{n} - \gamma \leq -\frac{\gamma}{2} \,,
\end{align*}
as $\gamma \leq \frac{1}{16 L \omega/n}$. Furthermore
\begin{align*}
\frac{\gamma^2}{n} + a \alpha^2 = \frac{\gamma^2}{n} \rbr*{1+ 4 \omega \alpha} \leq \frac{5 \gamma^2}{n} 
\end{align*}
as $\alpha \leq \frac{1}{1+\omega}$. Combining these bounds proves the claim.

\noindent Finally, without convexity, we start with the smoothness inequality, and derive similarly as in Lemma~\ref{lemma:qsgd_descent}
 \begin{align*}
 F_{t+1} \leq F_t - \gamma \frac{G_t}{2} + L\gamma^2 \frac{\sigma^2(1+\omega) + \zeta^2 \omega}{2n} + \gamma^2 \frac{L\omega}{2n} H'_t\,.
  \end{align*}
 Therefore, together with Lemma~\ref{lemma:Hnonconvex},
 \begin{align*}
  \Xi_{t+1} &\leq F_t + b\rbr*{1-\alpha + \frac{\gamma^2 L \omega}{b 2 n}} H'_t +  \rbr*{\alpha b Z^2 -\frac{\gamma}{2}} G_t + \alpha b \zeta^2 + \alpha^2 b \sigma^2 (1+\omega) + L\gamma^2 \frac{\sigma^2(1+\omega) + \zeta^2 \omega}{2n} \\
  &\leq \Xi_{t} + \frac{\gamma}{4} \rbr*{\frac{4\gamma \omega Z^2}{n} - 2}G_t + L\gamma^2 \frac{\sigma^2(2+\omega) + 2\zeta^2 \omega}{2n}
 \end{align*}
 for $b=\frac{\gamma^2 L \omega}{2\alpha n}$ and $\alpha = \frac{1}{1+\omega}$.
\end{proof}

The proof of Theorem~\ref{thm:diana} follows now from Lemma~\ref{lemma:diana2} with the help of the tools provided in Lemma~\ref{lemma:convergence-strong-convex} below for the strongly convex case, and with Lemma~\ref{lemma:convergence} for the non-convex case. For the convex ($\mu=0$) case, note that unrolling the expression from Lemma~\ref{lemma:diana2} (as in Lemma~\ref{lemma:convergence}) gives
\begin{align*}
 F_t &= \cO\rbr*{ \frac{\Psi_0}{\gamma} + \gamma \frac{\sigma^2(1+\omega)}{n} } \\
    &=  \cO\rbr*{ \frac{X_0}{\gamma} + \gamma \frac{\sigma^2(1+\omega) + \omega(1+\omega)\zeta_\star^2}{n} }
\end{align*}
by the observation $\Psi_0 := X_0 + a H_0$ with $a H_0 \leq \frac{4 \gamma^2 \omega}{\alpha n} \zeta_\star^2$. The proof of the claim now follows by the same steps as in Lemma~\ref{lemma:convergence}.

\section{D-EF-SGD.}
In this section we prove Theorem~\ref{thm:ef}. We follow closely~\cite{StichK19delays} and define a virtual sequence
\begin{align*}
 \tilde \xx_0 &:= \xx_0\,, &  \tilde\xx_{t+1} &:= \tilde\xx_t - \frac{\gamma}{n} \sum_{i \in [n]} \gg_t^i  \,,
\end{align*}
similar as in the main text in~\eqref{def:virtual}. Further, we will be using the notation $X_t := \Ef \norm{\xx_t - \xx^\star}^2$, $F_t := \Ef f(\xx_t)-f^\star$, $G_t := \Ef \norm{\nabla f(\xx_t)}^2$, $E_t:= \frac{1}{n}\sum_{i \in [n]} \Ef \norm{\ee_t^i}^2$.

\begin{lemma}[{\citet{StichK19delays}}]\label{lemma:Xn}
Let $f$ be $L$-smooth. If the stepsize $\gamma \leq \frac{1}{4L}$, then it holds for the iterates of Algorithm~\ref{alg:ef} and~\ref{alg:efnew}:
\begin{align}
 F_{t+1} &\leq F_t - \frac{\gamma}{4} G_t + \gamma^2 \frac{L \sigma^2}{2 n} + \gamma^3 \frac{ L^2}{2} E_t \label{eq:Gn}
\end{align}
and if $f$ is in addition $\mu$-convex,
\begin{align}
  X_{t+1} &\leq \left(1-\frac{\gamma \mu}{2} \right)  X_t - \frac{\gamma}{2}  F_t +  \frac{\gamma^2\sigma^2}{n} + 3\gamma^3 L  E_t \,. \label{eq:Xn}
\end{align}
\end{lemma}
\begin{proof}
First, we observe that the update applied to the virtual sequence in~\eqref{def:virtual} is an unbiased estimator of $\nabla f(\xx_t)$, with variance:
\begin{align*}
 \E_{\bxi_t} \norm[\Bigg]{\frac{1}{n}\sum_{i \in [n]} \gg_t^i - \nabla f(\xx_t)}^2 \stackrel{\eqref{def:noise}}{\leq} \frac{\sigma^2}{n}\,.
\end{align*}
With this observation, the proof follows directly from \cite{StichK19delays} 
with the observation that
\begin{align*}
 \norm{\tilde \xx_{t} - \xx_t}^2 &\stackrel{\eqref{eq:error}}{=} \norm[\Bigg]{\frac{\gamma}{n}\sum_{i \in [n]} \ee_{t}^i}^2 \stackrel{\eqref{eq:norm}}{\leq} \frac{\gamma^2}{n} \sum_{i \in [n]} \norm{\ee_t^i}^2\,. \qedhere
\end{align*}
\end{proof}

\begin{lemma}\label{lemma:En}
 It holds
\begin{align}
  E_{t+1} \leq \rbr*{1-\frac{\delta}{2}}  E_t +  \frac{2(1-\delta)}{\delta} \rbr*{\zeta^2 + Z^2G_t}   + (1-\delta) \sigma^2 \,. \label{eq:En}
\end{align}
\end{lemma}
\begin{proof}
From \cite{StichK19delays} it follows
\begin{align*}
 \E_{\bxi_t^i, \cC_\delta} \norm{\ee_{t+1}^i}^2 \leq \rbr*{1-\frac{\delta}{2}} \norm{\ee_t^i}^2 + \frac{2 (1-\delta)}{\delta} \norm{\nabla f_i(\xx_t)}^2 + (1-\delta)\sigma^2\,,
\end{align*}
The claim now follows by summing and averaging over $i \in [n]$.
\end{proof}

For the convergence proof, we can now either follow \cite{StichK19delays} again, or for a slightly simpler proof, we can combine Lemmas~\ref{lemma:Xn} and Lemma~\ref{lemma:En} together:
\begin{lemma}[Lyapunov function]
\label{lemma:ef_lyapunov}
Let $f$ be $L$-smooth and $\gamma \leq \frac{\delta}{4L(1+Z)}$. Then it holds
\begin{align*}
 \Xi_{t+1} \leq \Xi_t - \frac{\gamma}{8}  G_t + \gamma^2 \frac{L \sigma^2 }{2 n} + \gamma^3 \rbr*{\frac{2L^2 \zeta^2}{\delta^2} + \frac{L^2 \sigma^2}{\delta}}
\end{align*}
for $\Xi_t := F_t + b E_t$, for $b = \frac{\gamma^3 L^2}{\delta}$,
and
\begin{align*}
 \Psi_{t+1} \leq \rbr*{1-\min\cbr*{\gamma \mu/2, \delta/4}} \Psi_t  - \frac{\gamma}{4} G_t + \gamma^2 \frac{\sigma^2}{n} + \gamma^3\rbr*{\frac{24 L \zeta^2}{\delta^2} + \frac{12 L \sigma^2}{\delta}} 
\end{align*}
for $\Psi_t := X_t + a E_t$, with $a= \frac{12 \gamma^3 L}{\delta}$.
\end{lemma}
\begin{proof}
For convex functions, we first note that $G_t \stackrel{\eqref{eq:smoothL}}{\leq} 2LF_t$. Now
\begin{align*}
 \Psi_{t+1} &\stackrel{\eqref{eq:Xn},\eqref{eq:En}}{\leq} \rbr*{1-\frac{\gamma \mu}{2}} X_t + a\rbr*{1-\frac{\delta}{2} + \frac{3\gamma^3 L}{a} } + \rbr*{\frac{4a(1-\delta)L Z^2}{\delta} - \frac{\gamma}{2}} F_t  \\
 & \qquad \qquad + \frac{2 a (1-\delta)}{\delta} \zeta^2 + \frac{\gamma^2 \sigma^2}{n} + a(1-\delta)\sigma^2 \\
 &\leq (1-c) \Psi_t + \frac{\gamma}{4} \rbr*{\frac{192 \gamma^2 L^2 Z^2 }{\delta^2} - 2 } F_t + \gamma^2 \frac{\sigma^2}{n} + \gamma^3\rbr*{\frac{24 L \zeta^2}{\delta^2} + \frac{12 L \sigma^2}{\delta}} 
\end{align*}
with the choice $a=\frac{12 \gamma^3 L}{\delta}$, $c=\min\cbr*{\gamma\mu/2,\delta/4}$. Now the claim follows with $\gamma \leq \frac{\delta}{14 L Z}$.

\noindent For smooth functions,
\begin{align*}
\Xi_{t+1} \leq &\stackrel{\eqref{eq:Gn},\eqref{eq:En}}{\leq} F_t + a\rbr*{1-\frac{\delta}{2} + \frac{\gamma^3 L^2}{2b}} E_t + \rbr*{\frac{2b Z^2 }{\delta} - \frac{\gamma}{4} } G_t + \frac{2b \zeta^2}{\delta} + b\sigma^2 +  \frac{L\sigma^2}{2n} \\
& \leq \Xi_t + \frac{\gamma}{8} \rbr*{\frac{16 \gamma^2 Z^2}{\delta^2} - 2} G_t + \gamma^2 \frac{L \sigma^2 }{2 n} + \gamma^3 \rbr*{\frac{2L^2 \zeta^2}{\delta^2} + \frac{L^2 \sigma^2}{\delta}}
\end{align*}
for $b = \frac{\gamma^3 L^2}{\delta}$.
\end{proof}
\noindent As in the previous sections, the claims of the theorem follow from Lemma~\ref{lemma:ef_lyapunov} together with Lemmas~\ref{lemma:convergence-strong-convex} and~\ref{lemma:convergence} below.

\section{D-EF-SGD with bias correction}
\label{app:efsgdbias}
In this section we give the remaining proofs for the convergence of D-EF-SGD with bias correction for the convex and non-convex case.

{\update
We further remark that Algorithm~\ref{alg:efnew} is derived from Algorithm~\ref{alg:ec-diana}, with minor differences in the way the updates are compressed on the workers (line~\ref{ln:diff1} each) and how the parameters are updated on the sever (line~\ref{ln:diff2}  each).
}
\begin{figure*}[h!]
\begin{minipage}[t]{1\linewidth}
\centering
\begin{minipage}[t]{0.7\linewidth}
\begin{algorithm}[H]
\caption{EC-SGD-DIANA~\cite{Kovalev2019:personal,Gorbunov2020:ef}}
\label{alg:ec-diana}
\begin{algorithmic}[1]
\myState{\textbf{Input:} $\xx_0,\! \gamma,\! \cC_\delta,\! \cQ_\omega$, $\hh_0 \!=\! \hh_0^i \!=\! \ee_t^i \!= \! \0_d$, $\alpha \leq \frac{1}{1+\omega}$}
\For{$t=0,\dots,T-1$} \hfill $\triangledown$ worker side
 \myState{$\gg_t^i := \gg^i(\xx_t)$} \Comment{stochastic gradient}
 \myState{$\hat \Delta_t^i := \cC_\delta(\ee_t^i + \gg_t^i - \hh_t^i+\hh_t)$}
 \vspace{0.5em}
 \myState{$ \Delta_t^i :=\cQ_\omega(\gg_t^i-\hh_t^i)$} %
 \myState{$\ee_{t+1}^i := \ee_t^i + \gg_t^i - \hh_t^i + \hh_t - \hat \Delta_t^i$, $\hh_{t+1}^i := \hh_t^i + \alpha \Delta_t^i$}
 \vspace{0.5em}
 \myState{send to server: $\hat \Delta_t^i$, $\Delta_t^i$} \hfill $\triangledown$ server side
 \myState{$\xx_{t+1}:= \xx_t - \frac{\gamma}{n} \sum_{i=1}^n  \hat \Delta_t^i$} \label{ln:diff3} 
 \myState{$\hh_{t+1}:= \frac{1}{n}\sum_{i=1}^n \hh_{t+1}^i$} 
\EndFor
\end{algorithmic}
\end{algorithm}
\end{minipage}

\end{minipage}
\label{fig:dmitry}
\end{figure*}

\paragraph{Convex case.}
In the case when $\mu=0$, Lemma~\ref{lemma:convex} still applies. After unrolling the recursion (as in Lemma~\ref{lemma:convergence}), we obtain that
\begin{align*}
 F_t &= \cO \rbr*{\frac{\Psi_0}{\gamma} + \gamma \frac{\sigma^2}{n} + \gamma^2 \frac{1+\beta/\delta}{\delta}  L(1-\delta)\sigma^2 } \\
  &= \cO \rbr*{\frac{X_0}{\gamma} + \gamma \frac{\sigma^2}{n} + \gamma^2 \frac{\sigma^2 + \zeta_\star^2/\delta}{\delta}  L(1-\delta) }
\end{align*}
when plugging in $\Psi_0 = X_0 + a E_0 + b H_0$, $H_0 \leq \zeta_\star^2$ and $\beta=\delta$. Now the result follows by tuning $\gamma$ as in Lemma~\ref{lemma:convergence}.

\paragraph{Non-convex case.}
In the case when $f$ is only to be assumed $L$ smooth, the proof follows immediately by slight adaptations from tools that we have already developed above:

\begin{lemma}[{\citet[Lemma~8]{StichK19delays}}]\label{lemma:Xnonconvex}
Let $f$ be $L$-smooth. If the stepsize $\gamma \leq \frac{1}{2L}$, then it holds for the iterates of Algorithm:
\begin{align}
  F_{t+1} &\leq F_t - \frac{\gamma}{4} G_t + \frac{\gamma^2 L \sigma^2}{2 n} + \frac{\gamma^3 L^2}{2}  E_t \,. \label{eq:Xnonconvex}
\end{align}
\end{lemma}
\begin{proof}
\cite[Lemma~8]{StichK19delays} yields the result when resorting to the same observations as outline in the proof of Lemma~\ref{lemma:X} above.
\end{proof}

\noindent With the dissimilarity assumption, we can derive a new version of Lemma~\ref{lemma:E} without the need of the convexity assumption.
\begin{lemma}
\label{lemma:Enonconvex}
Let $f$ satisfy bounded dissimilarity~\eqref{def:zeta} for $\zeta^2,Z^2 \geq 0$. Then
\begin{align}
  E_{t+1} \leq \rbr*{1-\frac{\delta}{2}}  E_t +  \frac{4(1-\delta)}{\delta} \rbr*{\zeta^2 + ZG_t +  H'_t}    + (1-\delta) \sigma^2 \,, \label{eq:Enonconvex}
\end{align}
for $H'_t := \frac{1}{n}\sum_{i \in [n]} \Ef \norm{\nabla f_i(\xx_t)}^2$.
\end{lemma}
\begin{proof}
This readily follows from the proof of Lemma~\ref{lemma:E}. It remains to note that
\begin{align*}
 \frac{1}{n} \sum_{i \in [n]} \norm{\nabla f_i(\xx_t) - \hh_t^i}^2 &
 \stackrel{\eqref{eq:young}}{\leq} \frac{2}{n} \sum_{i \in [n]} \norm{\nabla f_i(\xx_t)}^2 + 2H'_t \stackrel{\eqref{def:zeta}}{\leq} 2 \zeta^2 + 2Z^2 \norm{\nabla f(\xx_t)}^2 + 2H'_t\,, 
\end{align*}
and to plug this estimate into~\eqref{eq:Econtinue}.
\end{proof}

\noindent Lastly, we show that $H'_t$ follows a similar recursion as $H_t$:
\begin{lemma}
\label{lemma:Hnonconvex}
Let $f$ satisfy bounded dissimilarity~\eqref{def:zeta} for $\zeta^2,Z^2 \geq 0$ and let $\hh_t^i$ be updated with an unbiased quantizer $\cQ_\omega$, and stepsize $\alpha \leq \frac{1}{1+\omega}$. Then
\begin{align}
  H'_{t+1} \leq (1-\alpha) H'_t +  \alpha (\zeta^2 + Z^2 G_t) + \alpha^2 (1+\omega) \sigma^2\,. \label{eq:Hnonconvex}
\end{align}
\end{lemma}
\begin{proof}
For the proof we note that the derivations in the proof of Lemma~\ref{lemma:Hunbiased} hold for arbitrary choice of $\hh_\star^i$, especially the choice $\hh_\star^i = \0_d$, up to equation~\eqref{eq:Hcontinue}. The claim now follows with~\eqref{def:zeta}.
\end{proof}

\noindent We can summarize the statements of Lemmas~\ref{lemma:Xnonconvex}--\ref{lemma:Hnonconvex} the following descent lemma.
\begin{lemma}[Lyapunov function]
\label{lemma:nonconvex}
Let $f$ be $L$-smooth and the stepsize $\gamma \leq \frac{1}{2L(1+4Z)}$. Then for $\beta = \delta$ it holds
\begin{align}
 \Psi_{t+1} \leq \Psi_t - \gamma \frac{G_t}{8} + \gamma^2 \frac{L \sigma^2}{2n} + \gamma^3 \frac{8 L^2 (1-\delta) \rbr*{\delta \sigma^2 + \zeta^2}}{\delta^2} \,,
\end{align}
for $\Psi_t := F_t + a E_t + b H'_t$ with $a=\frac{\gamma^3L^2}{\delta}$ and $b=\frac{4a(1-\delta)}{\delta}$.
\end{lemma}

Now the claimed convergence bound from Table~\ref{tab:results} follows by Lemma~\ref{lemma:convergence}, with the observation $\Psi_0 = F_0 + a E_0 + b H'_0 \leq F_0$ for the chosen initialization.

\section{Summation Lemmas}
In this section we repeat a few useful lemmas that have been (only slightly) adapted from other works.

\begin{lemma}[Based on Appendix A.2 of {\citet{Koloskova20decentralized}}]\label{lemma:convergence-strong-convex}
Let $(r_t)_{t\geq 0}$ and  $(s_t)_{t\geq 0}$ be sequences of positive numbers satisfying 
\begin{align}
r_{t+1} \leq 
\rbr*{1- \min \cbr*{\gamma A, F} } r_t - B\gamma s_t + C \gamma^2 + D \gamma^3 \,, \label{eq:rec1}
\end{align}
for some positive constants $A,B > 0$, $C,D \geq 0$, and for constant step-sizes $0 < \gamma \leq \frac{1}{E}$, for $E \geq 0$, and for parameter $0 < F \leq 1$. Then there exists a constant stepsize $\gamma \leq \frac{1}{E}$ such that
\begin{align*}
    \frac{B}{W_T}\sum_{t=0}^T w_t s_t  + \min\cbr*{A,\frac{F}{\gamma}} r_{T+1} \leq r_0 \rbr*{E+\frac{A}{F}}  \exp \left[- \min\cbr*{\frac{A}{E},F}(T+1) \right]+ \frac{ 2 C \ln \tau}{A(T+1)} + \frac{D \ln^2 \tau }{A^2 (T+1)^2}
\end{align*}
for $w_t := \rbr*{1- \min \cbr*{\gamma A, F} }^{-(t+1)}$, $W_T := \sum_{t=0}^T w_t$ and
\begin{align}
    \tau = \max\left\{\exp[1], \min\left\{\frac{A^2 r_0 (T+1)^2}{C}, \frac{A^3 r_0 (T+1)^3}{D} \right\} \right\}
\end{align}
\end{lemma}
\begin{remark}
\label{rem:convergence-strong-convex}
Lemma~\ref{lemma:convergence-strong-convex} establishes a bound of the order
\begin{align*}
 \tilde \cO \rbr*{r_0 \rbr*{E+\frac{A}{F}}  \exp \left[- \min\cbr*{\frac{A}{E},F}T \right] + \frac{C}{AT} + \frac{D}{A^2 T^2}} \,,
\end{align*}
that decreases with $T$. To ensure that this expression is less than $\epsilon$,
\begin{align*}
 T= \tilde \cO \rbr*{\frac{C}{ A \epsilon} + \frac{\sqrt{D}}{A \sqrt{\epsilon}} + \frac{1}{F}\log \frac{1}{\epsilon} + \frac{E}{A}\log \frac{1}{\epsilon} }  = \tilde \cO \rbr*{\frac{C}{ A \epsilon} + \frac{\sqrt{D}}{A \sqrt{\epsilon}} + \frac{1}{F} + \frac{E}{A} }
\end{align*}
steps are sufficient.
\end{remark}

\begin{proof}[Proof of Lemma~\ref{lemma:convergence-strong-convex}]
After rearranging and multiplying~\eqref{eq:rec1} by $w_t$ we obtain
\begin{align*}
   B w_t s_t \leq \frac{\rbr*{1- \min \cbr*{\gamma A, F} } w_t r_t}{\gamma} - \frac{w_t r_{t+1}}{\gamma} + \gamma C + \gamma^2 D\;.
\end{align*}
Observing that that $w_t \rbr*{1- \min \cbr*{\gamma A, F} } =w_{t-1}$ we  obtain a telescoping sum,
\begin{align*}
    \frac{B}{W_t} \sum_{t=0}^T w_t s_t &\leq \frac{\rbr*{1- \min \cbr*{\gamma A, F} } w_0 r_0 - w_{T} r_{T+1}}{\gamma W_T}  + \gamma C + \gamma^2 D = \frac{r_0}{\gamma W_T} - \frac{w_T r_{T+1}}{\gamma W_T}+ \gamma C + \gamma^2 D\,.
\end{align*}
Using that $W_T = w_{T} \sum_{t=0}^T \rbr*{1- \min \cbr*{\gamma A, F} }^t \leq \frac{w_T}{\min \cbr*{\gamma A, F}}$ and $W_T \geq w_T = (1-\min \cbr*{\gamma A, F})^{-(T+1)}$ we can simplify
\begin{align}
    \frac{B}{W_T} \sum_{t=0}^T w_t s_t + \min\cbr*{A,F/\gamma} r_{T+1} &\leq \frac{(1-\min \cbr*{\gamma A, F})^{T+1} r_0}{\gamma} + \gamma C + \gamma^2 D \notag \\ & \leq \frac{r_0}{\gamma} \exp \left[- \min\cbr*{\gamma A,F}(T+1) \right] + \gamma C + \gamma^2 D =: \Psi_T \label{eq:2432}
\end{align}
Now the lemma follows by tuning $\gamma$ in the same way as in~\cite[Lemma 2]{Stich19sgd} (slightly more carefully): 

\begin{itemize}
    \item If $\frac{\ln \tau}{A(T+1)} \leq \frac{1}{E}$ then we choose $\gamma = \frac{\ln \tau}{A(T+1)}$. With observing $\ln \tau \geq 1$ we obtain that
    \begin{align*}
        \Psi_T &\leq \underbrace{\frac{1}{\ln \tau} \max \left\{\frac{ C}{A(T+1)}, \frac{ D}{A^2(T+1)^2} \right\}}_{\text{Case: $\min\{\gamma A,F\} = \gamma A$} } + \underbrace{r_0 \frac{A}{F} \exp \left[- F (T+1)\right] }_{\text{Case: $\min\{\gamma A,F\} = F, \frac{1}{\gamma}\leq \frac{A}{F}$}}  +  \frac{C \ln \tau}{A(T+1)} + \frac{D \ln^2 \tau }{A^2 (T+1)^2} \\
        &\leq   r_0 \frac{A}{F} \exp \left[- \min\cbr*{\frac{A}{E},F}(T+1)\right] + \frac{2 C \ln \tau}{A(T+1)} + \frac{ 2 D \ln^2 \tau }{A^2 (T+1)^2}
    \end{align*}
    \item If otherwise $\frac{1}{E} \leq \frac{\ln \tau}{A(T+1)}$ and we pick $\gamma = \frac{1}{E}$ and get that 
    \begin{align*}
        \Psi_T &\leq r_0 E  \exp \left[- \min\cbr*{\frac{A}{E},F}(T+1)\right]  + \frac{C}{E} + \frac{D}{E^2} \\
        & \leq r_0 E  \exp \left[- \min\cbr*{\frac{A}{E},F}(T+1)\right] + \frac{C \ln \tau}{A(T+1)} + \frac{D \ln^2 \tau }{A^2 (T+1)^2} \,. \qedhere
    \end{align*}
\end{itemize}
\end{proof}

\begin{lemma}\label{lemma:convergence}
Let $(r_t)_{t\geq 0}$ and  $(s_t)_{t\geq 0}$ be sequences of positive numbers satisfying 
\begin{equation*}
r_{t+1} \leq 
r_t - B\gamma s_t + C \gamma^2 + D \gamma^3 \:, 
\end{equation*}
for some positive constants $B > 0$, $C, D \geq 0$ and step-sizes $0 < \gamma \leq \frac{1}{E}$, for $E \geq 0$. Then there exists a constant stepsize $\gamma \leq \frac{1}{E}$ such that
\begin{align}
\frac{B}{T+1} \sum_{t=0}^{T} s_t \leq \frac{E r_0}{T+1} + 2 D^{1/3} \left( \frac{r_0}{T+1} \right)^{2/3} + 2 \left( \frac{C r_0}{T+1} \right)^{1/2}  \,. \label{eq:253}
\end{align}
\end{lemma}
\begin{remark}\label{rem:convergence}
To ensure that the right hand side in~\eqref{eq:253} is less than $\epsilon > 0$,
\begin{align*}
 T= \cO \rbr*{ \frac{C}{\epsilon^2} + \frac{\sqrt{D}}{\epsilon^{3/2}} + \frac{E}{\epsilon} } \cdot r_0
\end{align*}
steps are sufficient.
\end{remark}
\begin{proof}[Proof of Lemma~\ref{lemma:convergence}]
Rearranging and dividing by $\gamma > 0$ gives
\begin{align*}
    B s_t \leq \frac{r_t}{\gamma} - \frac{r_{t+1}}{\gamma} + C \gamma + D \gamma^2
\end{align*}
and summing from $t=0$ to $T$ yields
\begin{align*}
    \frac{B}{T+1}  \sum_{t=0}^{T} s_t  \leq \frac{r_0}{\gamma (T+1)} + C \gamma + D \gamma^2\,.
\end{align*}
Now the claim follows by choosing $\gamma = \min \cbr*{\frac{1}{E}, \rbr*{\frac{r_0}{C(T+1)}}^{1/2}, \rbr*{\frac{r_0}{D(T+1)}}^{1/3}}$. See for instance~\cite[Lemma 15]{Koloskova20decentralized}.
\end{proof}

\end{document}

%% file: epfml_macros.tex
\usepackage{amsmath,amssymb,amsthm}
\usepackage{color,mathtools}
\usepackage[round]{natbib}

\usepackage{multirow}
\usepackage{enumitem}

\usepackage{bm}
\newcommand{\bxi}{\boldsymbol{\xi}}

\usepackage[hyperindex,breaklinks]{hyperref}

\usepackage{amssymb,amsmath,amsthm,dsfont}

\providecommand{\lin}[1]{\ensuremath{\left\langle #1 \right\rangle}}
  \providecommand{\R}{\mathbb{R}} % Reals
  \providecommand{\E}{{\mathbb E}}
  \providecommand{\Ef}{{\mathbf E }}
  \providecommand{\0}{\mathbf{0}}
  
  \renewcommand{\aa}{\mathbf{a}}
  \providecommand{\bb}{\mathbf{b}}

  \providecommand{\ee}{\mathbf{e}}
  
  \let\ggg\gg
  \renewcommand{\gg}{\mathbf{g}}
  
  \providecommand{\hh}{\mathbf{h}}

  \providecommand{\xx}{\mathbf{x}}
  \providecommand{\yy}{\mathbf{y}}

  \providecommand{\mV}{\mathbf{V}}

  \providecommand{\cC}{\mathcal{C}}

  \providecommand{\cO}{\mathcal{O}}
  
  \providecommand{\cQ}{\mathcal{Q}}
  
  \providecommand{\cS}{\mathcal{S}}

\RequirePackage[colorinlistoftodos,bordercolor=orange,backgroundcolor=orange!20,linecolor=orange,textsize=scriptsize]{todonotes}
\providecommand{\comment}[2]{\todo[inline,caption={}]{\textbf{#1: }#2}}%
\providecommand{\inlinecomment}[3]{%
  {\color{#1}#2: #3}}%
\newcommand\commenter[2]%
{%
  \expandafter\newcommand\csname i#1\endcsname[1]{\inlinecomment{#2}{#1}{##1}}
  \expandafter\newcommand\csname #1\endcsname[1]{\comment{#1}{##1}}
}

\newtheorem{lemma}{Lemma}

\newtheorem{remark}[lemma]{Remark}

\newtheorem{theorem}[lemma]{Theorem}

\DeclarePairedDelimiterX{\inp}[2]{\langle}{\rangle}{#1, #2}
\DeclarePairedDelimiterX{\abs}[1]{\lvert}{\rvert}{#1}
\DeclarePairedDelimiterX{\norm}[1]{\lVert}{\rVert}{#1}
\DeclarePairedDelimiterX{\cbr}[1]{\{}{\}}{#1} % curly bracket
\DeclarePairedDelimiterX{\rbr}[1]{(}{)}{#1} % round bracket
\DeclarePairedDelimiterX{\sbr}[1]{[}{]}{#1} % square bracket

\definecolor{mydarkblue}{rgb}{0,0.08,0.45}
\hypersetup{ %
    colorlinks=true,
    linkcolor=mydarkblue,
    citecolor=mydarkblue,
    filecolor=mydarkblue,
    }
\makeatletter
\newcommand{\myitem}[1]{%
\item[\textbf{(#1)}]\protected@edef\@currentlabel{#1}%
}
\makeatother